\documentclass{article}

\usepackage[utf8]{inputenc} 
\usepackage[T1]{fontenc}    
\usepackage{url}            
\usepackage{booktabs}       
\usepackage{amsfonts}       
\usepackage{nicefrac}       
\usepackage{microtype}      
\usepackage[nonatbib, final]{nips_2017}

\title{Unifying PAC and Regret: Uniform PAC Bounds for Episodic Reinforcement Learning}

\author{
  Christoph Dann\\
  Machine Learning Department\\
  Carnegie-Mellon University\\
  \texttt{cdann@cdann.net} \\
\And
Tor Lattimore\thanks{Tor Lattimore is now at DeepMind, London}\\
  \texttt{tor.lattimore@gmail.com} \\
\And
    Emma Brunskill\\
    Computer Science Department\\
    Stanford University\\
    \texttt{ebrun@cs.stanford.edu}
}

\usepackage[numbers]{natbib}
\usepackage[textsize=tiny]{todonotes}
\newcommand{\E}{\mathbb E}
\usepackage{wrapfig}
\usepackage{caption}

\usepackage{url}

\usepackage[utf8]{inputenc}
\usepackage{amsmath}
\usepackage{graphicx}
\usepackage{upgreek}
\usepackage{amsfonts}
\usepackage{amssymb}
\usepackage{amsthm}
\usepackage[mathscr]{euscript}
\usepackage{mathtools}
\newtheorem{thm}{Theorem}
\newtheorem{defn}{Definition}
\newtheorem{cor}{Corollary}

\newtheorem{lem}{Lemma}
\usepackage{xcolor}
\usepackage{nicefrac}
\usepackage{xr}
\usepackage{chngcntr}
\usepackage{apptools}
\usepackage[page, header]{appendix}
\AtAppendix{\counterwithin{lem}{section}}
\usepackage{titletoc}
\usepackage{enumitem}
\setlist[itemize]{leftmargin=1cm}
\setlist[enumerate]{leftmargin=1cm}

\definecolor{DarkRed}{rgb}{0.75,0,0}
\definecolor{DarkGreen}{rgb}{0,0.5,0}
\definecolor{DarkPurple}{rgb}{0.5,0,0.5}
\definecolor{Dark}{rgb}{0.5,0.5,0}
\definecolor{DarkBlue}{rgb}{0,0,0.7}
\usepackage[bookmarks, colorlinks=true, plainpages = false, citecolor = DarkBlue, urlcolor = blue, filecolor = black, linkcolor =DarkGreen]{hyperref}
\usepackage{breakurl}
\usepackage[ruled, vlined, linesnumbered]{algorithm2e}

\SetCommentSty{mycommfont}

\DeclareMathOperator*{\argmax}{arg\,max}

\allowdisplaybreaks[2]
\newcommand{\prob}{\mathbb P}

\newcommand{\Ex}{\mathbb E}

\newcommand{\indicator}[1]{\mathbb I\{ #1 \} }
\newcommand{\statespace}{\mathcal S}
\newcommand{\actionspace}{\mathcal A}
\newcommand{\saspace}{\statespace \times \actionspace}

\newcommand{\numS}{S}
\newcommand{\numA}{A}
\newcommand{\wmin}{w_{\min}}
\newcommand{\wminc}{w'_{\min}}
\newcommand{\range}{\operatorname{rng}}
\newcommand{\polylog}{\operatorname{polylog}}

\newcommand{\reals}{\mathbb R}

\newcommand{\set}[1]{\left\{#1\right\}}
\newcommand{\llnp}{\operatorname{llnp}}
\newcommand{\defeq}{:=}
\usepackage{xspace}
\newcommand{\algname}{UBEV\xspace}

\mathtoolsset{showonlyrefs=true}

\let\epsilon\varepsilon

\begin{document} 
\maketitle
\begin{abstract} 
Statistical performance bounds for reinforcement learning (RL) algorithms can be critical for high-stakes applications like healthcare.
This paper introduces a new framework for theoretically measuring the performance of such algorithms called \emph{Uniform-PAC}, which is
a strengthening of the classical Probably Approximately Correct (PAC) framework. In contrast to the PAC framework, the uniform version may be used
to derive high probability regret guarantees and so forms a bridge between the two setups that has been missing in the literature.
We demonstrate the benefits of the new framework for finite-state episodic MDPs with a
new algorithm that is Uniform-PAC and simultaneously achieves optimal regret \textit{and} PAC guarantees except for a factor of the horizon.

\end{abstract} 
\section{Introduction}
The recent empirical successes of deep reinforcement learning (RL) are tremendously exciting, but the performance of these approaches still varies significantly
across domains, each of which requires the user to solve a new tuning problem \citep{franccois2015discount}.
Ultimately we would like reinforcement learning algorithms that simultaneously perform well empirically and have strong theoretical guarantees. 
Such algorithms are especially important for high stakes domains like health care, education and customer service, where non-expert users demand excellent outcomes.

We propose a new framework for measuring the performance of reinforcement learning algorithms called Uniform-PAC. 
Briefly, an algorithm is Uniform-PAC if with high probability it simultaneously for all $\epsilon > 0$ selects an $\epsilon$-optimal policy
on all episodes except for a number that scales polynomially with $1/\epsilon$. 
Algorithms that are Uniform-PAC converge to an optimal policy
with high probability and immediately yield both PAC and high probability
regret bounds, which makes them superior to algorithms that come with only PAC or regret guarantees.
Indeed,
\begin{enumerate}
\item[(a)] Neither PAC nor regret guarantees imply convergence to optimal policies with high probability;
\item[(b)] $(\epsilon,\delta)$-PAC algorithms may be $\epsilon/2$-suboptimal in every episode;
\item[(c)] Algorithms with small regret may be maximally suboptimal infinitely often.
\end{enumerate}
Uniform-PAC algorithms suffer none of these drawbacks.
One could hope that existing algorithms with PAC or regret guarantees might be Uniform-PAC already, with only the analysis missing.
Unfortunately this is not the case and
modification is required to adapt these approaches to satisfy the new performance metric. 
The key insight for obtaining Uniform-PAC guarantees is to
leverage time-uniform concentration bounds such as the finite-time versions of the law of iterated logarithm, which
obviates the need for horizon-dependent confidence levels. 

We provide a new optimistic algorithm for episodic RL called UBEV 
that is Uniform PAC.  Unlike its predecessors, UBEV uses confidence intervals
based on the law of iterated logarithm (LIL) which hold uniformly over time.
They allow us to more tightly control the probability of failure events in which
the algorithm behaves poorly. Our 
analysis is nearly optimal according to the traditional metrics, with a linear dependence on the state
space for the PAC setting and square root dependence for the regret.
Therefore UBEV is a Uniform PAC algorithm with PAC bounds and high probability
regret bounds that are near optimal in the dependence on the length of the
episodes (horizon) and optimal in the state and action spaces cardinality as well as
the number of episodes. To our knowledge UBEV is the first algorithm with both
near-optimal PAC and regret guarantees.

\paragraph{Notation and setup.}
We consider episodic fixed-horizon MDPs with time-dependent dynamics, which can be formalized as
a tuple $M = (\statespace, \actionspace, p_R, P, p_0, H)$. The statespace
$\statespace$ and the actionspace $\actionspace$ are finite sets with cardinality $\numS$ and $\numA$. The 
agent interacts with the MDP in episodes of $H$ time steps each. 
At the beginning of each time-step $t \in [H]$ the agent observes a state $s_t$ and chooses an action $a_t$ based on a
policy $\pi$ that may depend on the within-episode time step ($a_t = \pi(s_t, t)$).
The next state is sampled from the $t$th transition kernel $s_{t+1} \sim
P(\cdot | s_t, a_t, t)$ and the initial state from $s_1 \sim p_0$. The agent then receives a reward drawn from a distribution $p_R(s_t, a_t, t)$ which can depend on $s_t, a_t$ and  $t$ 
with mean $r(s_{t}, a_t , t)$ determined by the reward function. The reward distribution $p_R$ is supported on $[0,1]$.\footnote{The reward may be allowed to depend on the next-state with no further effort in the proofs. The boundedness assumption
could be replaced by the assumption of subgaussian noise with known subgaussian
parameter.}
The value function from time step $t$ for policy $\pi$ is defined as
\begin{align*}
&V^\pi_{t}(s) 
\defeq \Ex\left[ \sum_{i=t}^H r(s_i, a_i, i) \bigg| s_t = s\right] = \sum_{s' \in \statespace}P(s'|s, \pi(s, t), t)  V^\pi_{t+1}(s') +r(s, \pi(s, t), t)\,.
\end{align*}
and the optimal value function is denoted by $V^{\star}_t$.  
In any fixed episode, the quality of a policy $\pi$ is
evaluated by the \emph{total expected reward} or \emph{return}
\begin{align*}
    \rho^\pi \defeq \Ex\left[ \sum_{i=t}^H r(s_i, a_i, i) \big| \pi \right] = p_0^\top V_1^{\pi}\,,
\end{align*}
which is compared to the \emph{optimal return} $\rho^{\star} = p_0^\top V_1^{\star}$. For this notation $p_0$ and the value functions $V^\star_t$, $V_1^{\pi}$ are interpreted as vectors of length $\numS$.
If an algorithm follows policy $\pi_k$ in episode $k$, then the
optimality gap in episode $k$ is $\Delta_k \defeq \rho^{\star} - \rho^{\pi_k}$ which is bounded by $\Delta_{\max} = \max_{\pi} \rho^{\star} - \rho^{\pi} \leq H$.
We let $N_\epsilon \defeq \sum_{k=1}^\infty \indicator{\Delta_k > \epsilon}$ be the number of $\epsilon$-errors and $R(T)$ be the regret after $T$ episodes: $R(T) \defeq \sum_{k=1}^T \Delta_k$.
Note that $T$ is the number of episodes and not total time steps (which is $HT$ after $T$ episodes) and $k$ is an episode index while $t$ usually denotes time indices within an episode.
The $\tilde O$ notation is similar to the usual $O$-notation but suppresses additional polylog-factors, that is $g(x) = \tilde O(f(x))$ iff there is a polynomial $p$ such that $g(x) = O(f(x) p(\log(x)))$.

\section{Uniform PAC and Existing Learning Frameworks}

We briefly summarize the most common performance measures used in the literature.
\begin{itemize}
\item \textit{$(\epsilon, \delta)$-PAC:} There exists a polynomial function $F_\textrm{PAC}(S, A, H, 1/\epsilon, \log(1/\delta))$ such that
\begin{align*}
\prob\left(N_\epsilon > F_{\textrm{PAC}}(S, A, H, 1/\epsilon, \log(1/\delta))\right) \leq \delta\,. 
\end{align*}
\item \textit{Expected Regret:} There exists a function $F_{\textrm{ER}}(S, A, H, T)$ such that $\E[R(T)] \leq F_{\textrm{ER}}(S, A, H, T)$.
\item \textit{High Probability Regret:} There exists a function $F_{\textrm{HPR}}(S, A, H, T, \log(1/\delta))$ such that 
\begin{align*}
\prob\left(R(T) > F_{\textrm{HPR}}(S, A, H, T, \log(1/\delta))\right) \leq \delta\,.
\end{align*}
\item \textit{Uniform High Probability Regret:} There exists a function $F_{\textrm{UHPR}}(S, A, H, T, \log(1/\delta))$ such that
\begin{align*}
\prob\left(\text{exists } T : R(T) > F_{\textrm{UHPR}}(S, A, H, T, \log(1/\delta))\right) \leq \delta\,.
\end{align*}
\end{itemize}
In all definitions the function $F$ should be polynomial in all arguments. For
notational conciseness we often omit some of the parameters of $F$ where the
context is clear.  The different performance guarantees are widely used (e.g.
PAC: \citep{Lattimore2012,Dann2015,Jiang2016,Strehl2009}, (uniform)
high-probability regret: \citep{Jaksch2010,Agarwal2014, Srinivas2010}; expected
regret: \citep{Audibert2009, Auer2000, Bubeck2012, AuerOrtner2005}).
Due to space constraints, we will not discuss Bayesian-style
performance guarantees that only hold in expectation with respect to a distribution over problem instances.
We will shortly discuss the limitations of the frameworks listed above, but first
formally define the Uniform-PAC criteria

\begin{defn}[Uniform-PAC]
An algorithm is Uniform-PAC for $\delta > 0$ if 
\begin{align}
\prob\left(\text{exists } \epsilon > 0 : \, N_\epsilon > F_{\textrm{UPAC}}\left(S, A, H, 1/\epsilon, \log(1/\delta)\right) \right) \leq \delta\,,
\end{align}
where $F_{\textrm{UPAC}}$ is polynomial in all arguments.
\end{defn}

All the performance metrics are functions of the distribution of the sequence of errors over the episodes $(\Delta_k)_{k \in \mathbb N}$.
Regret bounds are the integral of this sequence up to time $T$, which is a random variable. The expected regret is just the expectation of the integral, while
the high-probability regret is a quantile. PAC bounds are the quantile of the size of the superlevel set for a fixed level $\epsilon$.
Uniform-PAC bounds are like PAC bounds, but hold for all $\epsilon$ simultaneously.

\paragraph{Limitations of regret.}
Since regret guarantees only bound the integral of $\Delta_k$ over $k$, it does not
distinguish between making a few severe mistakes and many small mistakes. In fact, since regret bounds
provably grow with the number of episodes $T$, an algorithm that achieves optimal regret may still make infinitely many mistakes (of arbitrary quality, see proof of Theorem~\ref{thm:existingconv} below).
This is highly undesirable in high-stakes scenarios. For example in drug treatment optimization in healthcare, 
we would like to distinguish between infrequent severe complications 
(few large $\Delta_k$) and frequent minor side effects  
(many small $\Delta_k$). In fact, even with an optimal regret bound,
we could still serve infinitely patients with the worst possible treatment.

\paragraph{Limitations of PAC.}
PAC bounds limit the number of mistakes for a given accuracy level $\epsilon$, but is otherwise non-restrictive. That means an algorithm 
with $\Delta_k > \epsilon / 2$ for all $k$ almost surely might still be $(\epsilon,\delta)$-PAC. Worse, many algorithms designed to 
be $(\epsilon, \delta)$-PAC actually exhibit this behavior because they explicitly halt learning once an $\epsilon$-optimal policy has been found. The less widely used TCE (total cost of exploration) bounds \citep{Pazis2016} and KWIK guarantees~\citep{Li2010} suffer from the same issueand for conciseness are not discussed in detail.

\paragraph{Advantages of Uniform-PAC.}
The new criterion overcomes the limitations of PAC and regret guarantees by measuring the number of $\epsilon$-errors at every level simultaneously.
By definition, algorithms that are Uniform-PAC for a $\delta$ are $(\epsilon, \delta)$-PAC for all $\epsilon > 0$.
We will soon see that an algorithm with a non-trivial Uniform-PAC guarantee also has small regret with high probability. Furthermore, 
there is no loss in the reduction so that an algorithm with optimal Uniform-PAC guarantees also has optimal regret, at least in the episodic RL setting. 
In this sense Uniform-PAC is the missing bridge between regret and PAC. 
Finally, for algorithms based on confidence bounds, Uniform-PAC guarantees are usually obtained without much additional work by replacing standard concentration
bounds with versions that hold uniformly over episodes (e.g. using the law of the iterated logarithms).
In this sense we think Uniform-PAC is the new `gold-standard' of theoretical guarantees for RL algorithms.

\subsection{Relationships between Performance Guarantees}
Existing theoretical analyses usually focus exclusively on either the regret or PAC framework.
Besides occasional heuristic translations, Proposition~4 in \citep{Strehl2008} and Corollary~3 in \citep{Jaksch2010} are the only results relating a notion of PAC and regret, we are aware of. Yet the guarantees there are not widely used\footnote{The average per-step regret in \citep{Jaksch2010} is superficially a PAC bound, but does not hold over infinitely many time-steps and exhibits the limitations of a conventional regret bound. The translation to average loss in \citep{Strehl2008} comes at additional costs due to the discounted infinite horizon setting.} unlike the definitions given above which we now formally relate to each other. A simplified overview of the relations discussed below is shown in Figure~\ref{fig:relations}.

\begin{figure}
    \begin{center}
    \includegraphics[width=.8\textwidth]{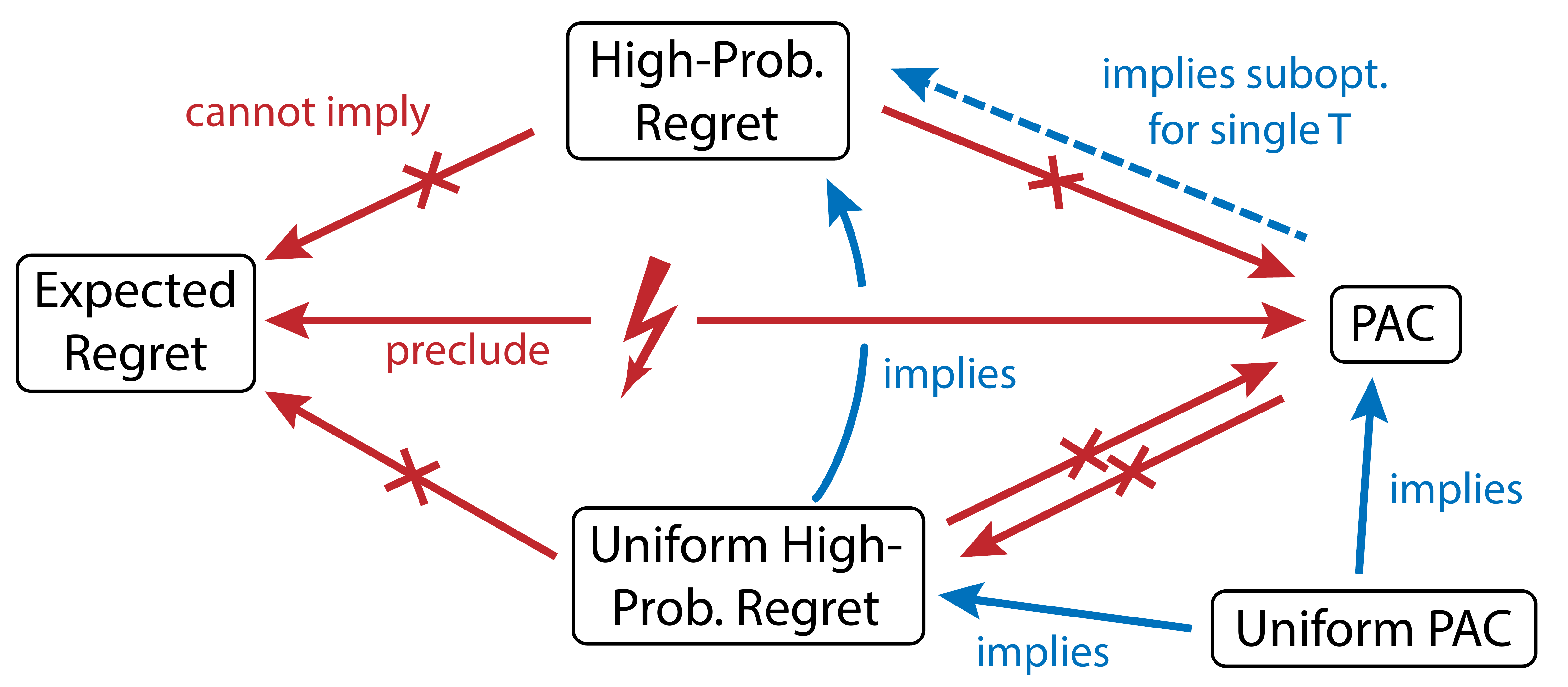}
\end{center}
\caption{Visual summary of relationship among the different learning frameworks: Expected regret (ER) and PAC preclude each other while the other crossed arrows represent only a \emph{does-not-implies} relationship. Blue arrows represent \emph{imply} relationships. For details see the theorem statements.}
    \label{fig:relations}
\end{figure}

\begin{thm}
    \label{thm:noerandpac}
No algorithm can achieve 
\begin{itemize}[itemsep=0pt]
    \item a sub-linear expected regret bound for all $T$ and 
    \item a finite $(\epsilon, \delta)$-PAC bound for a small enough $\epsilon$
 \end{itemize}
 simultaneously 
for all two-armed multi-armed bandits with Bernoulli reward distributions. This implies that such guarantees also cannot be satisfied simultaneously for all episodic MDPs.
\end{thm}

A full proof is in Appendix~\ref{sec:noerandpac}, but the intuition is simple. Suppose a two-armed Bernoulli bandit has mean rewards $\nicefrac{1}{2}+\epsilon$ and $\nicefrac{1}{2}$ respectively and the second arm is chosen at most $F < \infty$ 
times with probability at least $1 - \delta$, then one can easily show that in an alternative bandit with mean rewards $\nicefrac{1}{2}+\epsilon$ and $\nicefrac{1}{2}+2\epsilon$ there is a non-zero
probability that the second arm is played finitely often and in this bandit the expected regret will be linear.
Therefore, sub-linear expected regret is only possible if each arm is pulled infinitely often almost surely.

\begin{thm}
The following statements hold for performance guarantees in episodic MDPs: 
\begin{enumerate}
\item[(a)]
    If an algorithm satisfies a $(\epsilon, \delta)$-PAC bound with $F_{\textrm{PAC}} = \Theta(1 /
    \epsilon^2)$ then it satisfies for a specific $T=\Theta(\epsilon^{-3})$ a $F_{\textrm{HPR}} = \Theta(T^{2/3})$ bound.
    Further, there is an MDP and algorithm that satisfies the $(\epsilon, \delta)$-PAC bound $F_{\textrm{PAC}} = \Theta(1 / \epsilon^2)$ on that MDP and has regret $R(T) = \Omega(T^{2/3})$ on that MDP for any $T$.
    That means a $(\epsilon, \delta)$-PAC bound with $F_{\textrm{PAC}} = \Theta(1 /
    \epsilon^2)$ can only be converted to a high-probability regret bound with
    $F_{\textrm{HPR}} = \Omega(T^{2/3})$. 
\item[(b)] For any chosen $\epsilon, \delta > 0$ and $F_{\textrm{PAC}}$, there is an MDP and algorithm that satisfies the $(\epsilon, \delta)$-PAC bound $F_{\textrm{PAC}}$ on that MDP and has regret $R(T) = \Omega(T)$ on that MDP. That means 
    a $(\epsilon, \delta)$-PAC bound cannot be converted to a sub-linear uniform high-probability regret bound.
\item[(c)] For any  $F_{\textrm{UHPR}}(T, \delta)$ with  $F_{\textrm{UHPR}}(T, \delta) \rightarrow \infty$  as $T \rightarrow \infty$, there is an algorithm that satisfies that uniform high-probability regret bound on some MDP but makes infinitely many mistakes for any sufficiently small accuracy level $\epsilon > 0$ for that MDP. Therefore, a high-probability regret bound (uniform or not) cannot be converted to a finite $(\epsilon, \delta)$-PAC bound.
\item[(d)] For any $F_{\textrm{UHPR}}(T, \delta)$ there is an algorithm that satisfies that uniform high-probability regret bound on some MDP but suffers expected regret $\Ex R(T) = \Omega(T)$ on that MDP.
    \end{enumerate}
    \label{thm:existingconv}
\end{thm}
For most interesting RL problems including episodic MDPs the worst-case expected regret grows with $O(\sqrt{T})$.
The theorem shows that establishing an optimal high probability regret bound does not imply any finite PAC bound.
While PAC bounds may be converted to regret bounds, the resulting bounds are necessarily severely suboptimal with a rate of $T^{2/3}$. 
The next theorem formalises the claim that Uniform-PAC is stronger than both the PAC and high-probability regret criteria.

\begin{thm}
    \label{thm:unipac_properties}
    Suppose an algorithm is Uniform-PAC for some $\delta$ with $F_{\textrm{UPAC}} = \tilde O(C_1/\epsilon + C_2/\epsilon^2)$ where $C_1,C_2 > 0$ are constant in $\epsilon$, but may
    depend on other quantities such as $S$, $A$, $H$, $\log(1/\delta)$, then the algorithm
    \begin{enumerate}[itemsep=0mm]
        \item[(a)] converges to optimal policies with high probability:  $\prob(\lim_{k\to\infty} \Delta_k = 0) \geq 1 - \delta$.

        \item[(b)] is $(\epsilon, \delta)$-PAC with bound $F_{\textrm{PAC}} = F_{\textrm{UPAC}}$ for all $\epsilon$.
        \item[(c)] enjoys a high-probability regret at level $\delta$ with $F_{\textrm{UHPR}} = \tilde O(\sqrt{C_2 T} + \max\{C_1, C_2\})$.
    \end{enumerate}
\end{thm}
Observe that stronger uniform PAC bounds lead to stronger regret bounds and for
RL in episodic MDPs, an optimal uniform-PAC bound implies a uniform regret
bound.  To our knowledge, there are no existing approaches with PAC or regret
guarantees that are Uniform-PAC.  PAC methods such as MBIE, MoRMax,
UCRL-$\gamma$, UCFH, Delayed Q-Learning or Median-PAC all depend on advance
knowledge of $\epsilon$ and eventually stop improving their policies. Even when
disabling the stopping condition, these methods are not uniform-PAC as their
confidence bounds only hold for finitely many episodes and are eventually
violated according to the law of iterated logarithms.  Existing algorithms with
uniform high-probability regret bounds such as UCRL2 or UCBVI~\citep{Azar2017}
also do not satisfy uniform-PAC bounds since they use upper confidence bounds
with width $\sqrt{\log(T) / n}$ where $T$ is the number of observed episodes
and $n$ is the number of observations for a specific state and action. The
presence of $\log(T)$ causes the algorithm to try each action in each state
infinitely often.  One might begin to wonder if uniform-PAC is too good to be
true. Can \textit{any} algorithm meet the requirements?  We demonstrate in
Section~\ref{sec:analysis} that the answer is yes by showing that \algname has
meaningful Uniform-PAC bounds.  A key technique that allows us to prove these
bounds is the use of finite-time law of iterated logarithm confidence bounds
which decrease at rate $\sqrt{(\log \log n)/n}$.

\section{The \algname Algorithm}
\label{sec:algo}

The pseudo-code for the proposed \algname algorithm is given in Algorithm~\ref{alg:fhalg}. In each
episode it follows an optimistic policy $\pi_k$ that is computed by backwards
induction using a carefully chosen confidence interval on the transition
probabilities in each state. In line~\ref{lin:Qcomp} an optimistic estimate of
the Q-function for the current state-action-time triple is computed using the
empirical estimates of the expected next state value $\hat V_{\textrm{next}} \in \reals$ (given that the values at the next time are $\tilde V_{t+1}$) and expected
immediate reward $\hat r$ plus confidence bounds $(H-t) \phi$ and $\phi$.
We show in  Lemma~\ref{lem:optplaninterpret} in the appendix that the policy update in 
Lines~\ref{lin:optplan1}--\ref{lin:optplan2} finds an optimal solution to 
    $\max_{P', r', V', \pi'}  \,\, \Ex_{s \sim p_0} [V'_{1}(s)]$
subject to the constraints that for all $s \in \statespace, a \in \actionspace, t \in [H]$,
\begin{align} 
    & V'_{t}(s) = r(s, \pi'(s,t), t) + P'(s,\pi'(s,t), t)^\top V'_{t+1} \qquad\textrm{(Bellman Equation)}
    \label{eqn:prob1}
    \\
    & V'_{H+1} = 0, \quad P'(s,a,t) \in  \Delta_{\numS},\quad r'(s,a,t) \in [0,1]\\
    & |[(P' - \hat P_k)(s, a, t)]^\top V'_{t+1}| \leq \phi(s,a,t) (H-t)\\
    & |r'(s, a , t) - \hat r_k(s, a, t)| \leq \phi(s,a,t)
    \label{eqn:prob2}
\end{align}
where $(P' - \hat P_k)(s,a,t)$ is short for $P'(s,a,t) - \hat P_k(s,a,t) = P'(\cdot | s,a,t) - \hat P_k(\cdot | s,a,t)$ and 
\begin{align*}
    \phi(s,a,t) = \sqrt{\frac{2 \ln\ln \max\{e, n(s,a,t)\} + \ln(18 \numS \numA H /
    \delta)}{n(s,a,t)}} = O \left( \sqrt{\frac{\ln (\numS \numA H \ln (n(s,a,t)) / \delta)}{n(s,a,t)}} \right)
\end{align*}
is the width of a confidence bound with $e = \exp(1)$ and $\hat P_k(s' | s,a,t) = \frac{m(s', s,
a, t)}{n(s,a,t)}$ are the empirical transition probabilities and $\hat
r_k(s,a,t) = l(s,a,t) / n(s,a,t)$ the empirical immediate rewards (both at the beginning of the $k$th episode). 
Our algorithm is conceptually similar to other algorithms based on the optimism principle
such as MBIE \citep{Strehl2009}, UCFH \citep{Dann2015}, UCRL2 \citep{Jaksch2010} or
UCRL-$\gamma$ \citep{Lattimore2012} but there are several key differences:
\begin{itemize}
\item 
Instead of using confidence intervals over the transition kernel by itself, we incorporate the value function directly into the concentration analysis.
Ultimately this saves a factor of $S$ in the sample complexity, but the price is a more difficult analysis. Previously MoRMax~\citep{Szita2010} also used the idea of directly bounding the transition and value function, but in a very different algorithm that required discarding data and had a less tight bound. A similar technique has been used by \citet{Azar2017}.
\item Many algorithms update their policy less and 
    less frequently (usually when the number of samples doubles), and only finitely often in total. 
    Instead, we update the policy after every episode, which means that \algname immediately leverages new observations.
\item Confidence bounds in existing algorithms that keep improving the policy (e.g. \citet{Jaksch2010, Azar2017}) scale at a rate 
    $\sqrt{\log(k)/n}$ where $k$ is the number of episodes played so far and $n$ is the number of times the specific ($s,a,t$) has been observed. As the results of a brief empirical comparison in Figure~\ref{fig:expresults} indicate, this leads to slow learning (compare UCBVI\_1 and \algname's performance which differ essentially only by their use of different rate bounds). Instead the width of \algname's confidence bounds $\phi$ scales at rate $\sqrt{\ln \ln (\max\{e, n\})/n} \approx \sqrt{(\log \log n) / n}$ which is the best achievable rate and results in significantly faster learning.
\end{itemize}

\IncMargin{1.5em}
\begin{algorithm}[t]
\SetKwInOut{Inputa}{Input}
\Inputa{failure tolerance $\delta \in (0,1]$}
        $n(s,a, t) = l(s,a,t) = m(s', s, a, t) = 0; \quad \tilde V_{H+1}(s')\defeq 0 \quad \forall 
s, s' \in \statespace, a \in \actionspace, t \in [H]$\\
\For{$k=1, 2, 3, \dots$}{
    \tcc{Optimistic planning}
    \For{$t=H$ \KwTo $1$}
    {
    \label{lin:optplan1}
        \For{$s \in \statespace$}
        {
            \For{$a \in \actionspace$}
            {
                $\phi \defeq \sqrt{\frac{2 \ln \ln( \max\{e, n(s,a,t)\}) + \ln(18 \numS \numA H / \delta)}{n(s,a,t)}}$
                \label{lin:confbound}\tcp{confidence bound}
    $\hat r \defeq \frac{l(s,a,t)}{n(s,a,t)}; \quad
    \hat V_{\textrm{next}} \defeq \frac{m(\cdot, s, a, t)^\top \tilde V_{t+1}}{n(s,a,t)}$ \tcp{empirical estimates}
                $Q(a) \defeq \min\left\{1, \hat r + \phi \right\}
                + \min\left\{ \max \tilde V_{t+1},
                \hat V_{\textrm{next}} + (H - t) \phi \right\}$\label{lin:Qcomp}
             }
 $\pi_k(s, t) \defeq \argmax_{a} Q(a),\quad \tilde V_t(s) \defeq Q(\pi_k(s, t)) $
    \label{lin:optplan2}
}}
    \tcc{Execute policy for one episode}
    $s_1 \sim p_0$\; 
    \For{$t=1$ \KwTo $H$}
    {
        $a_t \defeq \pi_{k}(s_{t}, t), \,\,r_t \sim p_R(s_t, a_t, t)$ and $s_{t+1} \sim P(s_{t},a_t, t)$\\
        $n(s_t,a_t, t)\!+\!\!+; 
        \quad m(s_{t+1},s_t, a_t, t)\!+\!\!+;
        \quad l(s_t, a_t, t) +\!\!= r_t$ \tcp{update statistics}
    }
}

\caption{\algname ({\bf U}pper {\bf B}ounding the {\bf E}xpected Next State {\bf V}alue) Algorithm}
\label{alg:fhalg}
\end{algorithm}
\DecMargin{1.5em}
\begin{figure}[t]
    \begin{center}
\includegraphics[width=\columnwidth]{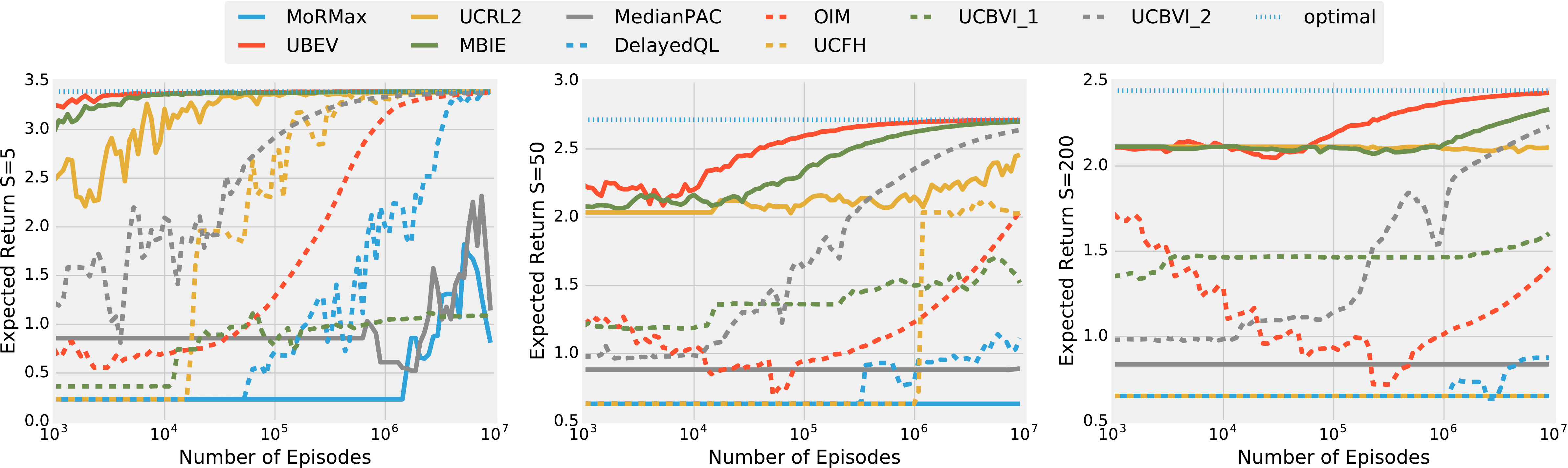}
    \end{center}
    \caption{Empirical comparison of optimism-based algorithms with frequentist regret or PAC bounds on a randomly generated MDP with $3$ actions, time horizon $10$ and $S=5, 50, 200$ states. All algorithms are run with parameters that satisfy their bound requirements. A detailed description of the experimental setup including a link to the source code can be found in Appendix~\ref{sec:expdetails}.}
    \label{fig:expresults}
\end{figure}

\section{Uniform PAC Analysis}
\label{sec:analysis}

We now discuss the Uniform-PAC analysis of \algname which results in the following Uniform-PAC and regret guarantee.
\begin{thm}
    Let $\pi_k$ be the policy of \algname in the $k$th episode.
    Then with probability at least $1 - \delta$ for all $\epsilon > 0$ jointly the number of episodes $k$ where the expected return from 
    the start state is not $\epsilon$-optimal 
    (that is $\Delta_k > \epsilon$) is at most
    \begin{align*}
      O\left(\frac{SAH^4}{\epsilon^2} \min\left\{1\! +\! \epsilon S^2A, S\right\} \polylog\!\left(\!A,S,H,\frac{1}{\epsilon},\frac{1}{\delta}\right)\!\!\right)\,.
    \end{align*}
    Therefore, with probability at least $1-\delta$ \algname converges to optimal policies and for all episodes $T$ has regret
    \begin{align*}
        R(T) = O\left(H^2(\sqrt{SAT} + S^3 A^2) \polylog(S,A,H,T)\right)\,.
    \end{align*}
    \label{thm:unipacupper}
\end{thm}
Here $\polylog(x \dots)$ is a function that can be bounded by a
polynomial of logarithm, that is, $\exists k,C : \polylog(x \dots) \leq \ln(x \dots)^k + C$. In Appendix~\ref{sec:lowerbound} we provide a lower bound on the sample
complexity that shows that if $\epsilon < 1/(S^2A)$, the Uniform-PAC bound is tight up
to log-factors and a factor of $H$.
To our knowledge, \algname is the first algorithm with both near-tight (up to $H$ factors) high probability regret and $(\epsilon,\delta)$ PAC bounds as well as the first algorithm with any nontrivial uniform-PAC bound.

Using Theorem~\ref{thm:unipac_properties} the convergence and regret bound follows immediately from the uniform PAC bound. After a discussion of the different confidence bounds allowing us to prove uniform-PAC bounds, we will provide a short proof sketch of the uniform PAC bound. 

\subsection{Enabling Uniform PAC With Law-of-Iterated-Logarithm Confidence Bounds} 

To have a PAC bound for all $\epsilon$ jointly, it is critical that \algname
continually make use of new experience.  If \algname stopped leveraging new
observations after some fixed number, it would not be able to 
distinguish with high probability among which of the remaining possible MDPs 
do or do not have optimal policies that are   
sufficiently optimal in the other MDPs. The algorithm therefore could potentially follow a policy that is not at least $\epsilon$-optimal for infinitely many episodes for a sufficiently small $\epsilon$.
To enable \algname to incorporate all new observations, the confidence bounds in \algname must hold for an infinite number of updates. 
We therefore require a proof that the total probability of all possible failure events (of the high confidence bounds not holding) is bounded by $\delta$, in order to obtain high probability guarantees. In contrast to prior $(\epsilon, \delta)$-PAC proofs that only consider a finite number of failure events (which is enabled by requiring an RL algorithm to stop using additional data), we must bound the probability 
of an infinite set of possible failure events. 

Some choices of confidence bounds will hold uniformly across all sample sizes but are not sufficiently tight for uniform PAC results. For example, the recent work by \citet{Azar2017} uses confidence intervals that shrink at a rate of $\sqrt{\frac{\ln T}{n}}$, where $T$ is the number of episodes, and $n$ is the number of samples of a $(s,a)$ pair at a particular time step. This confidence interval will hold for all episodes, but these intervals do not shrink sufficiently quickly and can even increase.
One simple approach for constructing confidence intervals that is sufficient for uniform PAC guarantees is to combine  bounds for fixed number of samples with a union bound allocating failure probability $\delta / n^2$ to the failure case with $n$ samples. This results in confidence intervals that shrink at rate $\sqrt{\nicefrac{1}{n}\ln n}$. Interestingly we know of no algorithms that do such in our setting. 

We follow a similarly simple but much stronger approach 
of using law-of-iterated logarithm (LIL) bounds that shrink at the better rate of $\sqrt{\nicefrac{1}{n}\ln \ln n}$. Such bounds have sparked recent interest in sequential decision making \citep{Jamieson2013,Balsubramani2016, Garivier2016, Massart2007,Garivier2011} but to the best of our knowledge we are the first to leverage them for RL. We prove several general LIL bounds in Appendix~\ref{sec:concentrationproofs} and explain how we use these results in our analysis in Appendix~\ref{sec:failureevents}. These LIL bounds are both sufficient to ensure uniform PAC bounds, and much tighter (and therefore will lead to much better performance) than  $\sqrt{\nicefrac{1}{n}\ln T}$ bounds.   Indeed, LIL have the tightest possible rate dependence on the number of samples $n$ for a bound that holds for all timesteps (though they are not tight with respect to constants). 

\subsection{Proof Sketch}
We now provide a short overview of our uniform PAC bound in Theorem~\ref{thm:unipacupper}. It follows the typical scheme for optimism based algorithms: we show that in each episode \algname follows a policy that is optimal with respect to the MDP $\tilde M_k$ that yields highest return in a set of MDPs $\mathcal M_k$ given by the constraints in Eqs.~\eqref{eqn:prob1}--\eqref{eqn:prob2} (Lemma~\ref{lem:optplaninterpret} in the appendix).
We then define a failure event $F$ (more details see below) such that on the complement $F^C$, the true MDP is in $\mathcal M_k$ for all $k$.  

Under the event that the true MDP is in the desired set, the $V^{\pi}_1 \leq V^{\star}_1 \leq \tilde V^{\pi_k}_1$, i.e., the value $\tilde V^{\pi_k}_1$ of $\pi_k$ in MDP $\tilde M_k$ is higher than the optimal value function of the true MDP $M$ (Lemma~\ref{lem:optimism}). Therefore, the optimality gap is bounded by $\Delta_k \leq p_0^\top(\tilde V^{\pi_k}_1 - V^{\pi_k}_1)$.
The right hand side this expression is then decomposed via a standard identity (Lemma~\ref{lem:valuediff}) as
\begin{align*}
    \sum_{t=1}^H  \sum_{(s,a) \in \saspace} w_{tk}(s,a)((\tilde P_{k}-P)(s, a, t))^\top \tilde V^{\pi_k}_{t+1}
    & +\sum_{t=1}^H  \sum_{(s,a) \in \saspace} w_{tk}(s,a)(\tilde r_{k}(s, a, t) - r(s, a, t)),
\end{align*}
where $w_{tk}(s,a)$ is the probability that when following policy $\pi_k$ in
the true MDP we encounter $s_t = s$ and $a_t = a$. The quantities $\tilde P_k$,
$\tilde r_k$ are the model parameters of the optimistic MDP $\tilde M_k$
For the sake of conciseness, we ignore the second term above in the following which
can be bounded by $\epsilon / 3$ in the same way as the first.  We further decompose the
first term as 
\begin{align}
    & \sum_{\mathclap{\substack{t\in[H] \\ (s,a) \in L_{tk}^c}}} w_{tk}(s,a)((\tilde P_{k} - P)(s, a, t))^\top \tilde V^{\pi_k}_{t+1}
    \label{eqn:decomp1_algt}
    \\
    +& \sum_{\mathclap{\substack{t\in[H] \\ (s,a) \in L_{tk}}}} w_{tk}(s,a)((\tilde P_{k} - \hat P_k)(s, a, t))^\top \tilde V^{\pi_k}_{t+1}
        + \sum_{\mathclap{\substack{t\in[H] \\ (s,a) \in L_{tk}}}} w_{tk}(s,a)((\hat P_{k} - P)(s, a, t))^\top \tilde V^{\pi_k}_{t+1}
    \label{eqn:decomp1_main}
\end{align}
where 
$
L_{tk} = \left\{ (s,a) \in \saspace \, : \, w_{tk}(s,a) \geq \wmin = \frac{\epsilon}{3HS^2} \right\}
$ is the
set of state-action pairs with non-negligible visitation probability. The value of $\wmin$ is chosen so that
\eqref{eqn:decomp1_algt} is bounded by $\epsilon / 3$.
Since $\tilde V^{\pi_k}$ is the optimal solution of the optimization
problem in Eq.~\eqref{eqn:prob1}, we can bound 
\begin{align}
|((\tilde P_{k} -
&\hat P_k)(s, a, t))^\top \tilde V^{\pi_k}_{t+1}| \leq \phi_k(s,a,t) (H-t) = O\left(\sqrt{\frac{H^2 \ln \left( \ln (n_{tk}(s,a))/ \delta \right)}{n_{tk}(s,a)}}
\right)\,,
\label{eq:decomp1_algt_dec}
\end{align}
where $\phi_k(s,a,t)$ is the value of $\phi(s,a,t)$ and $n_{tk}(s,a)$
the value of $n(s,a,t)$ right before episode $k$.
Further we decompose
\begin{align}
    & |((\hat P_{k} - P)(s, a, t))^\top \tilde V^{\pi_k}_{t+1}| \leq 
    \|(\hat P_{k} - P)(s, a, t)\|_1  \|\tilde V^{\pi_k}_{t+1}\|_\infty
    \leq 
    O\left(\sqrt{\frac{\numS H^2\ln \frac{\ln n_{tk}(s,a)}{\delta}}{n_{tk}(s,a)}}\right)\,,
    \label{eqn:123dd}
\end{align}
where the second inequality follows from a standard concentration bound used in the definition of the failure event $F$ (see below).
Substituting this and \eqref{eq:decomp1_algt_dec} into \eqref{eqn:decomp1_main} leads to
\begin{align}
\eqref{eqn:decomp1_main} \leq O\left(\sum_{t=1}^H  \sum_{s,a \in
    L_{tk}} w_{tk}(s,a) \sqrt{\frac{\numS H^2 \ln (\ln (n_{tk}(s,a))/\delta)}{n_{tk}(s,a)}}
\right).\label{eqn:a1123}
\end{align}
On $F^C$ it also holds that 
$n_{tk}(s,a) \geq \frac 1 2 \sum_{i < k} w_{ti}(s,a) - \ln \frac{9SAH}{\delta}$ and so on \emph{nice episodes} where each $(s,a) \in L_{tk}$ with significant probability $w_{tk}(s,a)$ also had significant probability in the past, i.e., $\sum_{i < k} w_{ti}(s,a) \geq 4 \ln \frac{9SA}{\delta}$, it holds that $n_{tk}(s,a) \geq \frac 1 4 \sum_{i < k} w_{ti}(s,a)$. Substituting this into \eqref{eqn:a1123}, we can use a careful pidgeon-hole argument laid out it Lemma~\ref{lem:mainratelemma} in the appendix to show that this term is bounded by $\epsilon / 3$ on all but 
$O(A S^2 H^4 / \epsilon^2 \polylog(A, S, H, 1/\epsilon, 1/ \delta))$ nice episodes. Again using a pidgeon-hole argument, one can show that all but at most $O(S^2 A H^3 / \epsilon \ln (SAH / \delta))$ episodes are nice. Combining both bounds, we get that on $F^C$ the optimality gap $\Delta_k$ is at most $\epsilon$ except for at most  $O(A S^2 H^4 / \epsilon^2 \polylog(A, S, H, 1/\epsilon, 1/ \delta))$ episodes.

We decompose the failure event into multiple components. 
In addition to the events $F^{N}_k$ that a $(s,a,t)$ triple has been observed few times compared to its visitation probabilities in the past, i.e.,  $n_{tk}(s,a) < \frac 1 2 \sum_{i < k} w_{ti}(s,a) - \ln \frac{9SAH}{\delta}$ as well as a conditional version of this statement, the failure event $F$ contains events where empirical estimates of the immediate rewards, the expected optimal value of the successor states and the individual transition probabilites are far from their true expectations. For the full definition of $F$ see Appendix~\ref{sec:failureevents}. $F$ also contains event $F^{L1}$ we used in Eq.~\eqref{eqn:123dd} defined as
    \begin{align*}
\resizebox{.97\hsize}{!}{$
    \left\{ \exists k, s, a, t\, : \, \|\hat P_k(s, a, t) - P(s, a, t)\|_1 
    \geq 
    \sqrt{\frac{4} {n_{tk}(s, a)} 
        \left(2 \llnp(n_{tk}(s, a))+ \ln \frac{18 \numS \numA H(2^\numS - 2)}{\delta}\right)
    } \right\}.$
}
\end{align*}
It states that the L1-distance of the empirical transition probabilities to the
true probabilities for any $(s,a,t)$ in any episode $k$ is too large and we
show that $\prob(F^{L1}) \leq 1 - \delta / 9$ using a uniform version of the
popular bound by \citet{Weissman2003} which we prove in
Appendix~\ref{sec:concentrationproofs}. We show in similar manner that the
other events in $F$ have small probability uniformly for all episodes $k$ so that $\prob(F) \leq \delta$.
Together this yields the uniform PAC bound in Thm.~\ref{thm:unipacupper} using the second term in the $\min$. 

With a more refined analysis that avoids the use of H\"older's inequality in \eqref{eqn:123dd} and a stronger notion of nice episodes called friendly episodes we obtain the bound with the first term in the $\min$. However, since a similar analysis has been recently released \citep{Azar2017}, we defer this discussion to the appendix.
\subsection{Discussion of \algname Bound}

The (Uniform-)PAC bound for \algname in Theorem~\ref{thm:unipacupper} is never worse than $\tilde O(\numS^2\numA
H^4/\epsilon^2)$, which improves on the similar MBIE algorithm by a factor
of $H^2$ (after adapting the discounted setting for which MBIE was analysed to our
setting). For $\epsilon < 1/(\numS^2\numA)$ our bound has a linear
dependence on the size of the state-space and depends on $H^4$, which is a
tighter dependence on the horizon than MoRMax's $\tilde O(\numS \numA H^6 /
\epsilon^2)$, the best sample-complexity bound with linear dependency
$\numS$ so far. 

Comparing UBEV's regret bound to the ones of UCRL2~\citep{Jaksch2010} and REGAL~\citep{Bartlett2009}
requires care because
(a) we measure the regret over entire episodes and (b) our transition dynamics are time-dependent within each episode, which effectively increases the state-space by a factor of $H$.
Converting the bounds for UCRL2/REGAL to our setting yields a regret bound of order $SH^2\sqrt{AHT}$. Here, the diameter is $H$, the state space increases by $H$ due to time-dependent transition dynamics and an additional $\sqrt{H}$ is gained by stating the regret in terms of episodes $T$ instead of time steps. Hence, UBEV’s bounds are better by a factor of $\sqrt{SH}$.
Our bound matches the recent regret bound for episodic RL by \citet{Azar2017} in the $S$, $A$ and $T$ terms but not in $H$.  \citet{Azar2017} has regret bounds that are optimal in $H$ but their algorithm is not uniform PAC, due to the characteristics we outlined in Section 2.

\section{Conclusion}

The Uniform-PAC framework strengthens and unifies the PAC and high-probability regret 
performance criteria for reinforcement learning in episodic MDPs. The newly proposed
algorithm is Uniform-PAC, which as a side-effect means it is the first algorithm that is
both PAC and has sub-linear (and nearly optimal) regret.
Besides this, the use of law-of-the-iterated-logarithm confidence bounds in RL algorithms for MDPs provides a practical and
theoretical boost at no cost in terms of computation or implementation
complexity.

This work opens up several immediate research questions for future work. The
definition of Uniform-PAC and the relations to other PAC and regret notions
directly apply to multi-armed bandits and contextual bandits as special cases
of episodic RL, but not to infinite horizon reinforcement learning. An extension
to these non-episodic RL settings is highly desirable. Similarly, a version of
the \algname algorithm for infinite-horizon RL with linear state-space sample
complexity would be of interest. More broadly, if theory is ever to say
something useful about practical algorithms for large-scale reinforcement
learning, then it will have to deal with the unrealizable function
approximation setup (unlike the tabular function representation setting
considered here), which is a major long-standing open challenge.
\\
\\
\noindent{\textbf{Acknowledgements}}. We appreciate the support of a NSF CAREER award and a gift from Yahoo. 

\bibliography{cdann_mendeley,manual}

\begin{thebibliography}{26}
\providecommand{\natexlab}[1]{#1}
\providecommand{\url}[1]{\texttt{#1}}
\expandafter\ifx\csname urlstyle\endcsname\relax
  \providecommand{\doi}[1]{doi: #1}\else
  \providecommand{\doi}{doi: \begingroup \urlstyle{rm}\Url}\fi

\bibitem[Fran{\c{c}}ois-Lavet et~al.(2015)Fran{\c{c}}ois-Lavet, Fonteneau, and
  Ernst]{franccois2015discount}
Vincent Fran{\c{c}}ois-Lavet, Rapha{\"e}l Fonteneau, and Damien Ernst.
\newblock How to discount deep reinforcement learning: Towards new dynamic
  strategies.
\newblock In \emph{NIPS 2015 Workshop on Deep Reinforcement Learning}, 2015.

\bibitem[Lattimore and Hutter(2014)]{Lattimore2012}
Tor Lattimore and Marcus Hutter.
\newblock {Near-optimal PAC bounds for discounted MDPs}.
\newblock In \emph{Theoretical Computer Science}, volume 558, 2014.

\bibitem[Dann and Brunskill(2015)]{Dann2015}
Christoph Dann and Emma Brunskill.
\newblock {Sample Complexity of Episodic Fixed-Horizon Reinforcement Learning}.
\newblock In \emph{Neural Information Processing Systems}, 2015.

\bibitem[Jiang et~al.(2017)Jiang, Krishnamurthy, Agarwal, Langford, and
  Schapire]{Jiang2016}
Nan Jiang, Akshay Krishnamurthy, Alekh Agarwal, John Langford, and Robert~E
  Schapire.
\newblock {Contextual Decision Processes with Low Bellman Rank are
  PAC-Learnable}.
\newblock In \emph{International Conference on Machine Learning}, 2017.

\bibitem[Strehl et~al.(2009)Strehl, Li, and Littman]{Strehl2009}
Alexander~L Strehl, Lihong Li, and Michael~L Littman.
\newblock {Reinforcement Learning in Finite MDPs : PAC Analysis}.
\newblock \emph{Journal of Machine Learning Research}, 10:\penalty0 2413--2444,
  2009.

\bibitem[Jaksch et~al.(2010)Jaksch, Ortner, and Auer]{Jaksch2010}
Thomas Jaksch, Ronald Ortner, and Peter Auer.
\newblock {Near-optimal Regret Bounds for Reinorcement Learning}.
\newblock \emph{Journal of Machine Learning Research}, 11:\penalty0 1563--1600,
  2010.

\bibitem[Agarwal et~al.(2014)Agarwal, Hsu, Kale, Langford, Li, and
  Schapire]{Agarwal2014}
Alekh Agarwal, Daniel Hsu, Satyen Kale, John Langford, Lihong Li, and Robert~E.
  Schapire.
\newblock {Taming the Monster: A Fast and Simple Algorithm for Contextual
  Bandits}.
\newblock In \emph{Journal of Machine Learning Research}, volume~32, 2014.

\bibitem[Srinivas et~al.(2012)Srinivas, Krause, Kakade, and
  Seeger]{Srinivas2010}
Niranjan Srinivas, Andreas Krause, Sham~M. Kakade, and Matthias~W. Seeger.
\newblock {Information-theoretic regret bounds for Gaussian process
  optimization in the bandit setting}.
\newblock In \emph{IEEE Transactions on Information Theory}, volume~58, 2012.

\bibitem[Audibert et~al.(2009)Audibert, Munos, and
  Szepesv{\'{a}}ri]{Audibert2009}
Jean~Yves Audibert, R{\'{e}}mi Munos, and Csaba Szepesv{\'{a}}ri.
\newblock {Exploration-exploitation tradeoff using variance estimates in
  multi-armed bandits}.
\newblock \emph{Theoretical Computer Science}, 410\penalty0 (19):\penalty0
  1876--1902, 2009.

\bibitem[Auer(2000)]{Auer2000}
Peter Auer.
\newblock {Using upper confidence bounds for online learning}.
\newblock \emph{Proceedings 41st Annual Symposium on Foundations of Computer
  Science}, pages 270--293, 2000.

\bibitem[Bubeck and Cesa-Bianchi(2012)]{Bubeck2012}
S{\'{e}}bastien Bubeck and Nicol{\`{o}} Cesa-Bianchi.
\newblock {Regret Analysis of Stochastic and Nonstochastic Multi-armed Bandit
  Problems}.
\newblock \emph{arXiv.org}, cs.LG\penalty0 (1):\penalty0 138, 2012.

\bibitem[Auer and Ortner(2005)]{AuerOrtner2005}
Peter Auer and Ronald Ortner.
\newblock {Online Regret Bounds for a New Reinforcement Learning Algorithm}.
\newblock In \emph{Proceedings 1st Austrian Cognitive Vision Workshop}, 2005.

\bibitem[Pazis and Parr(2016)]{Pazis2016}
Jason Pazis and Ronald Parr.
\newblock {Efficient PAC-optimal Exploration in Concurrent , Continuous State
  MDPs with Delayed Updates}.
\newblock In \emph{AAAI Conference on Artificial Intelligence}, 2016.

\bibitem[Li et~al.(2011)Li, Littman, Walsh, and Strehl]{Li2010}
Lihong Li, Michael~L. Littman, Thomas~J. Walsh, and Alexander~L. Strehl.
\newblock {Knows what it knows: A framework for self-aware learning}.
\newblock \emph{Machine Learning}, 82\penalty0 (3):\penalty0 399--443, nov
  2011.

\bibitem[Strehl and Littman(2008)]{Strehl2008}
Alexander~L. Strehl and Michael~L. Littman.
\newblock {An analysis of model-based Interval Estimation for Markov Decision
  Processes}.
\newblock \emph{Journal of Computer and System Sciences}, 74\penalty0
  (8):\penalty0 1309--1331, 2008.

\bibitem[Azar et~al.(2017)Azar, Osband, and Munos]{Azar2017}
Mohammad~Gheshlaghi Azar, Ian Osband, and R{\'{e}}mi Munos.
\newblock {Minimax Regret Bounds for Reinforcement Learning}.
\newblock In \emph{International Conference on Machine Learning}, 2017.

\bibitem[Szita and Szepesv{\'{a}}ri(2010)]{Szita2010}
Istv{\`{a}}n Szita and Csaba Szepesv{\'{a}}ri.
\newblock {Model-based reinforcement learning with nearly tight exploration
  complexity bounds}.
\newblock In \emph{International Conference on Machine Learning}, 2010.

\bibitem[Jamieson et~al.(2013)Jamieson, Malloy, Nowak, and
  Bubeck]{Jamieson2013}
Kevin Jamieson, Matthew Malloy, Robert Nowak, and S{\'{e}}bastien Bubeck.
\newblock {lil' UCB : An Optimal Exploration Algorithm for Multi-Armed
  Bandits}.
\newblock 2013.

\bibitem[Balsubramani and Ramdas(2016)]{Balsubramani2016}
Akshay Balsubramani and Aaditya Ramdas.
\newblock {Sequential Nonparametric Testing with the Law of the Iterated
  Logarithm}.
\newblock In \emph{Uncertainty in Artificial Intelligence}, 2016.

\bibitem[Garivier et~al.(2016)Garivier, Kaufmann, and Lattimore]{Garivier2016}
Aur{\'{e}}lien Garivier, Emilie Kaufmann, and Tor Lattimore.
\newblock {On Explore-Then-Commit Strategies}.
\newblock In \emph{Advances in Neural Information Processing Systems}, 2016.

\bibitem[Massart(2007)]{Massart2007}
Pascal Massart.
\newblock {Concentration inequalities and model selection}.
\newblock \emph{Lecture Notes in Mathematics}, 1896, 2007.

\bibitem[Garivier and Cappe(2011)]{Garivier2011}
Aurelien Garivier and Olivier Cappe.
\newblock {The KL-UCB Algorithm for Bounded Stochastic Bandits and Beyond}.
\newblock In \emph{Conference on Learning Theory}, 2011.

\bibitem[Weissman et~al.(2003)Weissman, Ordentlich, Seroussi, Verdu, and
  Weinberger]{Weissman2003}
Tsachy Weissman, Erik Ordentlich, Gadiel Seroussi, Sergio Verdu, and Marcelo~J
  Weinberger.
\newblock {Inequalities for the L 1 Deviation of the Empirical Distribution}.
\newblock Technical report, 2003.
\newblock URL
  \url{http://www.hpl.hp.com/techreports/2003/HPL-2003-97R1.pdf?origin=publicationDetail}.

\bibitem[Bartlett and Tewari(2009)]{Bartlett2009}
Peter~L. Bartlett and a.~Tewari.
\newblock {REGAL: A regularization based algorithm for reinforcement learning
  in weakly communicating MDPs}.
\newblock \emph{Proceedings of the Twenty-Fifth Conference on Uncertainty in
  Artificial Intelligence}, pages 35--42, 2009.

\bibitem[Boucheron et~al.(2013)Boucheron, Lugosi, and Massart]{Boucheron2013}
Stephane Boucheron, Gabor Lugosi, and Pascal Massart.
\newblock \emph{{Concentration Inequalities - A Nonasymptotic Theory of
  Independence}}.
\newblock Oxford University Press, 2013.
\newblock ISBN 978-0-19-953525-5.

\bibitem[Durrett(2010)]{Durrett2010}
Rick Durrett.
\newblock \emph{{Probability - Theory and Examples}}.
\newblock Cambridge University Press, 4 edition, 2010.
\newblock ISBN 978-0-521-76539-8.

\end{thebibliography}
\bibliographystyle{unsrtnat-nourl}
\newpage
\appendix

\counterwithin{lem}{section}
\counterwithin{thm}{section}
\counterwithin{cor}{section}

\renewcommand\appendixpagename{Appendices of Unifying PAC and Regret: Uniform PAC Bounds for Episodic Reinforcement Learning}
\appendixpage
\startcontents[sections]
\printcontents[sections]{l}{1}{\setcounter{tocdepth}{2}}
\newpage

\section{Framework Relation Proofs}
\subsection{Proof of Theorem~\ref{thm:noerandpac}}
\label{sec:noerandpac}
\begin{proof}
    We will use two episodic MDPs, $M_1$ and $M_2$, which are essentially
    2-armed bandits and hard to distinguish to prove this statement. Both
    MDPs have one state, horizon $H=1$, and two actions $\actionspace = \{1,
    2\}$. For a fixed $\alpha > 0$, the rewards are Bernoulli($1/2 + \alpha / 2$) distributed for actions $1$ in both
    MDPs. Playing action $2$ in $M_1$ gives Bernoulli($1/2$) rewards and action $2$ in $M_2$ gives Bernoulli($1/2
    + \alpha$) rewards.

    Assume now that an algorithm in MDP $M_1$ with nonzero probability plays the suboptimal action only at most $N$ times in total, i.e.,
    $\prob_{M_1}(n_2 \leq N) \geq \beta$ where $n_2$ is the number of times action $2$ is played and $ \infty > N > 0, \beta > 0$.
    Then
    \begin{align}
        \prob_{M_1}(n_2 \leq N) = \Ex_{M_1}\left[ \indicator{n_2 \leq N} \right] = 
        \Ex_{M_2}\left[ \frac{\prob_{M_1}(Y_\infty)}{\prob_{M_2}(Y_\infty)}\indicator{n_2 \leq N} \right]
    \end{align}
where $Y_k = (A_1, R_1, A_2, R_2, \dots A_k, R_k)$ denotes the entire sequence of observed rewards $R_i$ and action indices $A_i$ after $k$ episodes. Since $\prob_{M_1}(A_k | Y_{k-1}) = \prob_{M_2}(A_k | Y_{k-1})$ and $\prob_{M_1}(R_k | A_k=1, Y_{k-1}) = \prob_{M_2}(R_k | A_k=1 , Y_{k-1})$ and 
\begin{align}
    \frac{\prob_{M_1}(R_k | A_k=2, Y_{k-1})}{\prob_{M_2}(R_k | A_k=2 , Y_{k-1})} 
    \leq \max\left\{\frac{1/2}{1/2 + \alpha}, \frac{1/2}{1/2 - \alpha}\right\}
    = \frac{1}{1 - 2\alpha}
\end{align}
the likelihood ratio of $Y_\infty$ is upper bounded by $(1 + 2 \alpha)^N$ if the second action has been chosen at most $N$ times. Hence
    \begin{align}
       \prob_{M_2}\left[ n_2 \leq N \right] &= \frac{(1 - 2\alpha)^N}{(1 - 2 \alpha)^N} 
\Ex_{M_2}\left[ \indicator{n_2 \leq N} \right] 
\geq (1 - 2\alpha)^N \Ex_{M_2}\left[ \frac{\prob_{M_1}(Y_\infty)}{\prob_{M_2}(Y_\infty)}\indicator{n_2 \leq N} \right] \\ \geq&  (1 - 2\alpha)^N \beta > 0
    \end{align}
Therefore, the regret for $M_2$ is for $T$ large enough
$\Ex_{M_2} R(T) \geq (T-N) \beta (1-2\alpha)^N \alpha / 2 = O(T)$.
Hence, for the algorithm to ensure sublinear regret for $M_2$, it has to play the suboptimal action for $M_1$ infinitely often with probability $1$. This however implies that the algorithm cannot satisfy any finite PAC bound for accuracy $\epsilon < \alpha / 2$.
\end{proof}

\subsection{Proof of Theorem~\ref{thm:existingconv}}
\begin{proof}
\textbf{PAC Bound to high-probability regret bound:}
Consider a fixed $\delta > 0$ and PAC bound with $F_{\textrm{PAC}}= \Theta(1 / \epsilon^2)$. Then there is a $C > 0$ such that the following algorithm satisfies the PAC bound. The algorithm uses the worst possible policy with optimality gap $H$ in all episodes on some event $E$ and in the first $C / \epsilon^2$ episodes on the complimentary event $E^C$.
For the remaining episodes on $E^C$ it follows a policy with optimality gap $\epsilon$. 
The probability of $E$ is $\delta$. The regret of the algorithm on $E$ is $R(T) = TH$ and on $E^C$ it is $R(T) = \min\{T, C/\epsilon^2\} H + \min\{T - C/\epsilon^2, 0 \} \epsilon$. For $T \geq C/\epsilon^2$, on any event the regret of this algorithm is at least
\begin{align}
R(T) = \frac{CH}{\epsilon^2} + \left(T - \frac{C}{\epsilon^2}\right)\epsilon = T\epsilon + \frac{C(H-\epsilon)}{\epsilon^2}.
\label{eqn:minreg1}
\end{align}
The quantity 
\begin{align}
\frac{R(T)}{T^{2/3}} = \frac{C(H- \epsilon)}{T^{2/3}\epsilon^2} + \epsilon T^{1/3}
\end{align}
takes its minimum at $T = \frac{C(H-\epsilon)}{\epsilon^3}$  with a positive value and hence $R(T) = \Omega(T^{2/3})$. Therefore a PAC bound with rate $1 / \epsilon^2$ implies at best a high-probability regret bound of order $O(T^{2/3})$ and is only tight at $T = \Theta(1 / \epsilon^3)$. Furthermore,  by looking at Equation~\eqref{eqn:minreg1}, we see that for any fixed $\epsilon$, there is an algorithm that has uniform high-probability regret that is $\Omega(T)$. 

\textbf{PAC Bound to uniform high-probability regret bound:}
Consider a fixed $\delta > 0$ and $\epsilon > 0$ and a PAC bound $F_{\textrm{PAC}}$ that evaluates to some value $N$ for parameter $\epsilon$. The algorithm uses the worst possible policy with optimality gap $H$ in all episodes on some event $E$ and in the first $N$ episodes on the complimentary event $E^C$.
For the remaining episodes on $E^C$ it follows a policy with optimality gap $\epsilon$. 
The probability of $E$ is $\delta$. The regret of the algorithm on $E$ is $R(T) = TH$ and on $E^C$ it is $R(T) = \min\{T, N\} H + \min\{T -N, 0 \} \epsilon$. For $T \geq N$, on any event the regret of this algorithm is at least
\begin{align}
    R(T) = NH + \left(T - N \right)\epsilon = T\epsilon + H(T-N) = \Omega(T).
\label{eqn:minreg2}
\end{align}

\textbf{Uniform high-probability regret bound to PAC bound:}
Consider an MDP such that at least one suboptimal policy exists with optimality
gap $\epsilon > 0$. Further let $L(T)$  be a nondecreasing function with
$F_{\textrm{UHPR}}(T) \geq L(T)$ and $L(T) \rightarrow \infty$ as $T
\rightarrow \infty$. Then the algorithm plays the optimal policy except for
episodes $k$ where $\lfloor L(k-1) / \epsilon \rfloor \neq \lfloor L(k) / \epsilon \rfloor$. This
algorithm satisfies the regret bound but makes infinitely many
$\nicefrac{\epsilon}{2}$-mistakes with probability $1$.

\textbf{Uniform high-probability regret bound to expected regret bound:}
Consider an MDP such that at least one suboptimal policy exists with optimality
gap $\epsilon > 0$. Consider an algorithm that with probability $\delta$ always plays the suboptimal policy and with probability $1 - \delta$ always plays the optimal policy.
This algorithm satisfies the uniform high-probability regret bound but suffers regret
$\Ex R(T) = \delta \epsilon T = \Omega(T)$.
\end{proof}

\subsection{Proof of Theorem~\ref{thm:unipac_properties}}
\label{sec:pacregconv}
\begin{proof}
    \textbf{Convergence to optimal policies:} The convergence to the set of optimal policies follows directly by using the definition of limits on the $\Delta_k$ sequence for each outcome in the high-probability event where the bound holds.\\
    \textbf{$(\epsilon, \delta)$-PAC:} Due to sub-additivity of probabilities, we have
    \begin{align}
        &\prob\left( N_\epsilon > F_{\textrm{PAC}}\left(\frac 1 \epsilon, \log \frac 1 \delta\right)\right)
        \leq
        \prob\left( \bigcup_{\epsilon'}\left\{N_{\epsilon'} > F_{\textrm{PAC}}\left(\frac 1 {\epsilon'}, \log \frac 1 \delta\right)\right\}\right)\\
        = &
        \prob\left( \exists \epsilon' \,:\, N_{\epsilon'} > F_{\textrm{PAC}}\left(\frac 1 {\epsilon'}, \log \frac 1 \delta\right)\right) \leq \delta.
    \end{align}
    \textbf{High-Probability  Regret Bound:} This part is proved separately in Theorem~\ref{thm:uniformpactoregret} below.
\end{proof}

\begin{figure}
    \centering\includegraphics[width=.5\textwidth]{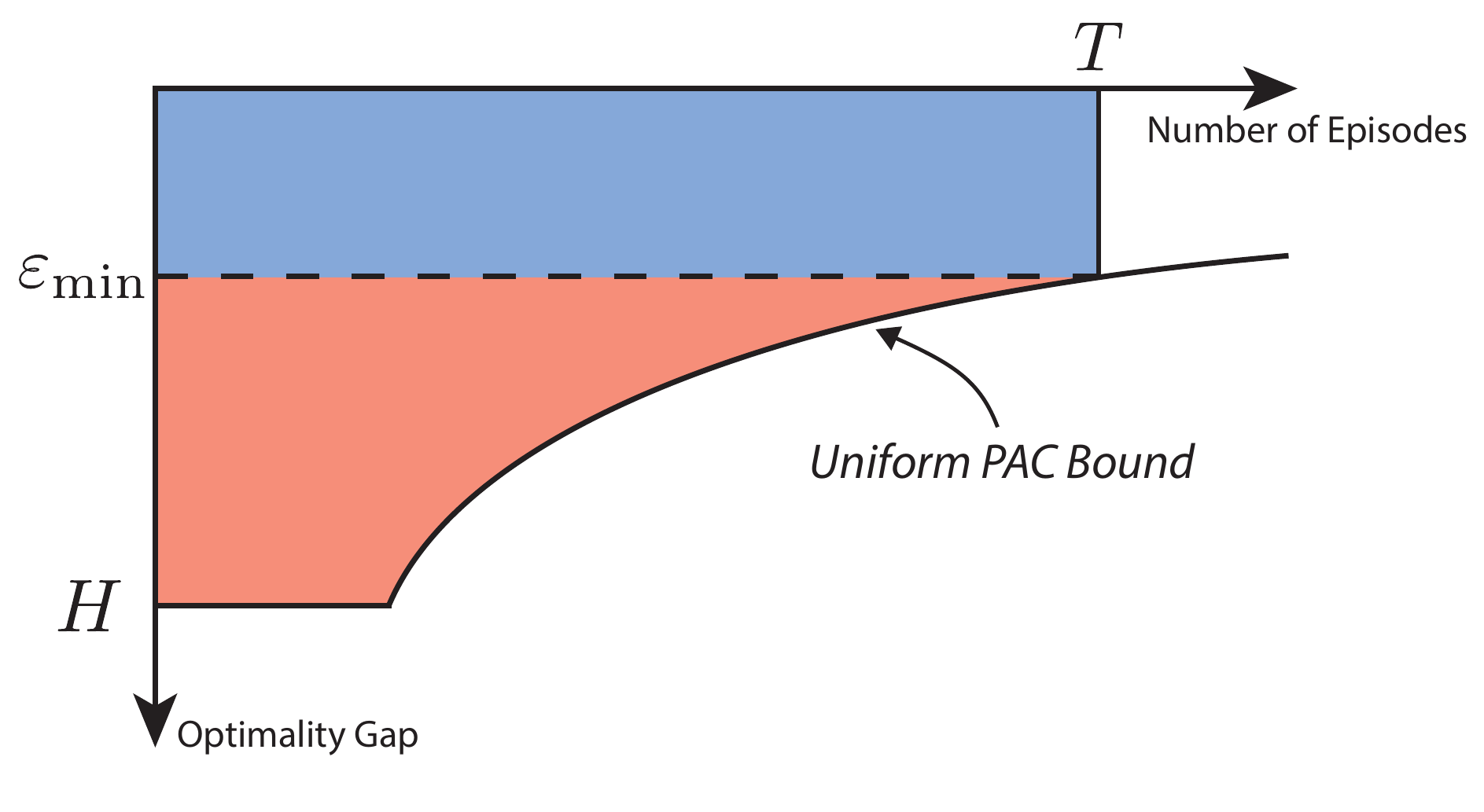}
        \caption{Relation of PAC-bound and Regret; The area of the shaded regions are a bound on the regret after $T$ episodes.}
        \label{fig:regret}
\end{figure}
\begin{thm}[Uniform-PAC to Regret Conversion Theorem]
    Assume on some event $E$ an algorithm follows for all $\epsilon$ an
    $\epsilon$-optimal policy $\pi_k$, i.e., $\Delta_k \leq
    \epsilon$, on all but at most 
    \begin{align}
        \frac{C_1 }{\epsilon}\left(\ln \frac {C_3}
    \epsilon\right)^k + \frac{C_2}{\epsilon^2}\left(\ln \frac {C_3}
    \epsilon\right)^{2k}
\end{align}
episodes where $C_1 \geq C_2 \geq 2$ and $C_3 \geq \max\{H, e\}$ and $C_1, C_2, C_3$ do not depend on
    $\epsilon$ . Then this algorithm has on this event a regret of 
    \begin{align}
        R(T) \leq (\sqrt{C_2 T}  + C_1 ) \polylog(T, C_3, C_1) = O(\sqrt{C_2 T} \polylog(T, C_3, C_1, H))
    \end{align}
    for all number of episodes $T$.
    \label{thm:uniformpactoregret}
\end{thm}
\begin{proof}
    The mistake bound $g(\epsilon) = \frac{C_1 }{\epsilon}\left(\ln \frac {C_3}
    \epsilon\right)^k + \frac{C_2}{\epsilon^2}\left(\ln \frac {C_3}
\epsilon\right)^{2k} \leq T$ is monotonically decreasing for $\epsilon \in (0,
H]$.
For a given $T$ large enough, we can therefore find an $\epsilon_{\min} \in (0,
H]$ such that $g(\epsilon) \leq T$ for all $\epsilon \in (\epsilon_{\min}, H]$.
The regret $R(T)$ of the algorithm can then be bounded as follows
\begin{align}
    R(T) \leq T \epsilon_{\min} + \int_{\epsilon_{\min}}^H g(\epsilon) d\epsilon.
\end{align}
This bound assumes the worst case where first the algorithm makes the worst
mistakes possible with regret $H$ and subsequently less and less severe
mistakes controlled by the mistake bound. For a better intuition, see
Figure~\ref{fig:regret}.

We first find a suitable $\epsilon_{\min}$.
Define $y = \frac 1 \epsilon \left( \ln \frac {C_3} \epsilon\right)^k$ then since $g$ is monotonically decreasing, it is sufficient
to find a $\epsilon$ with $g(\epsilon) \leq T$. That is equivalent to 
$C_1 y + C_2 y^2 \leq T$ for which
\begin{align}
   \frac 1 \epsilon \left( \ln \frac {C_3} \epsilon\right)^k =  y \leq \frac{C_1}{2 C_2} + \frac{\sqrt{C_1^2 + 4 T C_2}}{2 C_2} =: a
\end{align}
is sufficient. We set now
\begin{align}
    \epsilon_{\min} =& \frac{\ln(C_3 a)^k}{a} = \frac{2C_2}{C_1 + \sqrt{C_1^2 + 4TC_2}} \left( \ln \frac{(C_1 + \sqrt{C_1^2 + 4 TC_2})C_3}{2 C_2} \right)^k
\end{align}
which is a valid choice as
\begin{align}
    \frac 1 {\epsilon_{\min}} \left( \ln \frac {C_3} {\epsilon_{\min}}\right)^k 
    =& \frac{a}{\ln(C_3 a)^k} \left( \ln \frac{C_3 a}{\ln(C_3a)^k} \right)^k
    = \frac{a}{\ln(C_3 a)^k} \left( \ln(C_3 a) - k\ln \ln(C_3a) \right)^k \\
    \leq & \frac{a}{\ln(C_3 a)^k} \left( \ln(C_3 a) \right)^k = a.
\end{align}
We now first bound the regret further as
\begin{align}
    R(T) \leq& T \epsilon_{\min} + \int_{\epsilon_{\min}}^H g(\epsilon) d\epsilon
    \leq T \epsilon_{\min} 
    + C_1 \left(\ln \frac {C_3}
    {\epsilon_{\min}}\right)^k\int_{\epsilon_{\min}}^H 
    \frac{1}{\epsilon}
     d\epsilon
    + C_2 \left(\ln \frac {C_3}
    {\epsilon_{\min}}\right)^{2k}\int_{\epsilon_{\min}}^H 
    \frac{1}{\epsilon^2}
     d\epsilon\\
     = & T \epsilon_{\min} 
    + C_1 \left(\ln \frac {C_3}
    {\epsilon_{\min}}\right)^k
    \ln \frac{H}{\epsilon_{\min}}
    + C_2 \left(\ln \frac {C_3}
    {\epsilon_{\min}}\right)^{2k}
    \left[ \frac{1}{\epsilon_{\min}} - \frac 1 H \right]
\end{align}
and then use the choice of $\epsilon_{\min}$ from above to look at each of the terms in this bound individually. In the following bounds we extensively use the fact $\ln(a+b) \leq \ln(a) + \ln(b) = \ln(ab)$ for all $a,b \geq 2$ and that $\sqrt{a + b} \leq \sqrt{a} + \sqrt{b}$ which holds for all $a,b \geq 0$.
\begin{align}
    T\epsilon_{\min} &= \frac{2 TC_2}{C_1 + \sqrt{C_1^2 + 4 TC_2}} \left( \ln \frac{C_3(C_1 + \sqrt{C_1^2 + 4 TC_2})}{2 C_2} \right)^k\\
   &\leq
    \frac{2 TC_2}{\sqrt{4 TC_2}} \left( \ln C_3 + \ln C_1  + \ln C_1 + \ln \frac{2\sqrt{TC_2}}{2C_2} \right)^k\\
   &\leq
    \sqrt{TC_2} \left( \ln (C_3 C_1^2 \sqrt T) \right)^k
\end{align}
Now for a $C \geq 0$ we first look at
\begin{align}
    \ln \frac{C}{\epsilon_{\min}} 
     = &
    \ln C + \ln \frac{C_1 + \sqrt{C_1^2 + 4TC_2}}{2 C_2} - k \ln \ln \frac{C_3(C_1 + \sqrt{C_1^2 + 4 TC_2})}{2 C_2}\\
    \leq &
    \ln C + \ln \frac{C_1 + \sqrt{C_1^2 + 4TC_2}}{2 C_2} \\
    \leq &
    \ln C + \ln C_1 +  \ln C_1 + \ln \frac{\sqrt{4TC_2}}{2 C_2}\\
    \leq &
    \ln (C C_1^2 \sqrt{T})
\end{align}
where the first inequality follows from the fact that $\frac{C_3(C_1 + \sqrt{C_1^2 + 4 TC_2})}{2 C_2} \geq \frac{C_3 2C_1}{2 C_2} \geq e$.
Hence, we can bound
\begin{align}
   C_1 \left( \ln \frac{ C_3 }{\epsilon_{\min}} \right)^k \ln \frac H {\epsilon_{\min}}
   \leq C_1 \left( \ln (C_3 C_1^2 \sqrt{T}) \right)^k \ln (H C_1^2 \sqrt{T}).
\end{align}
Now since
\begin{align}
    \frac{1}{\epsilon_{\min}} =& \frac{C_1 + \sqrt{C_1^2 + 4TC_2}}{2C_2} 
    \left( \ln \frac{C_3(C_1 + \sqrt{C_1^2 + 4 TC_2})}{2 C_2} \right)^{-k}
    \leq \frac{C_1}{C_2} + \sqrt{\frac{T}{C_2}} 
\end{align}
we get
\begin{align}
    C_2 \left( \ln \frac{ C_3 }{\epsilon_{\min}} \right)^{2k} \left[ \frac 1 {\epsilon_{\min}} - \frac 1 H \right]
    \leq &
    C_2 \left( \ln (C_3 C_1^2 \sqrt{T})\right)^{2k} \left[ \frac{C_1}{C_2} + \sqrt{\frac{T}{C_2}}  \right]\\
    \leq &
    \left( \ln (C_3 C_1^2 \sqrt{T})\right)^{2k} \left[ C_1 + \sqrt{T C_2}  \right].
\end{align}
As a result we can conclude that $R(T) \leq (\sqrt{C_2 T}  + C_1 ) \polylog(T, C_3, C_1, H) = O(\sqrt{C_2 T} \polylog(T, C_3, C_1, H))$.
\end{proof}

\section{Experimental Details}
\label{sec:expdetails}

We generated the MDPs with $\numS=5, 50, 200$ states, $\numA=3$ actions and $H=10$
timesteps as follows: The transition probabilities $P(s, a, t)$ were sampled
independently from $\textrm{Dirichlet}\left(\frac{1}{10}, \dots
\frac{1}{10}\right)$ and the rewards were all deterministic with their value
$r(s,a,t)$ set to $0$ with probability $85\%$ and set uniformly at random in
$[0,1]$ otherwise. This construction results in MDPs that have concentrated but
non-deterministic transition probabilities and sparse rewards.

Since some algorithms have been proposed assuming the rewards $r(s,a,t)$ are
known and we aim for a fair comparison, we assumed for all algorithms that the
immediate rewards $r(s,a,t)$ are known and adapted the algorithms accordingly.
For example, in \algname, the $\min\left\{1, \frac{l(s,a,t)}{\max\{1,
n(s,a,t)\}} + \phi \right\}$ term was replaced by the true known rewards
$r(s,a,t)$ and the $\delta$ parameter in $\phi$ was scaled by $9/7$ accordingly
since the concentration result for immediate rewards is not necessary in this
case.  We used $\delta = \frac 1 {10}$ for all algorithms and $\epsilon = \frac
1 {10}$ if they require to know $\epsilon$ beforehand.

We adapted MoRMax, UCRL2, UCFH, MBIE, MedianPAC, Delayed Q-Learning and OIM to
the episodic MDP setting with time-dependent transition dynamics by using
allowing them to learn time-dependent dynamics and use finite-horizon planning.
We did adapt the confidence intervals and but did not re-derive the constants
for each algorithm. When in doubt we opted for smaller constants typically
resulting better performance of the competitors.  We further replaced the range
of the value function $O(H)$ by the observed range of the optimistic next state
values in the confidence bounds. 
We also reduced the number of
episodes used in the delays by a factor of $\frac{1}{1000}$ for MoRMax and Delayed Q-Learning and by $10^{-6}$ for UCFH
because they would otherwise not have performed a single policy update even for
$\numS = 5$ within the 10 million episodes we considered.  This scaling
violates their theoretical guarantees but at least shows that the methods work
in principle.

The performance reported in Figure~\ref{fig:expresults} are the expected return of
the current policy of each algorithm averaged over $1000$ episodes. The figure
shows a single run of the same randomly generated MDP but the results are
representative. We reran this experiments with different random seeds and
consistently obtained qualitatively similar results.

Source code for the experiments including concise but efficient implementations
of the algorithms is available at \url{https://github.com/chrodan/FiniteEpisodicRL.jl}.

\section{PAC Lower Bound}
\label{sec:lowerbound}
\begin{thm}
    There exist positive constants $c$, $\delta_0 > 0$, $\epsilon_0 > 0$ such
    that for every $\epsilon \in (0,
    \epsilon_0)$, $\numS \geq 4, \numA \geq 2$ and for every algorithm A that and $n \leq \frac{c \numA \numS H^3}{\epsilon^2}$ there is a
    fixed-horizon episodic MDP $M_{hard}$ with time-dependent transition
    probabilities and $\numS$ states and $\numA$ actions so that returning an $\epsilon$-optimal policy after $n$ episodes is at most $1 - \delta_0$.
    That implies that no algorithm can have a PAC guarantee better than 
        $\Omega\left( \frac{\numA \numS H^3}{\epsilon^2} \right)$
    for sufficiently small $\epsilon$.
\label{thm:lower_bound}
\end{thm}
Note that this lower bound on the sample complexity of any method in episodic MDPs with time-dependent dynamics applies to the arbitrary but fixed $\epsilon$ PAC
bound and therefore immediately to the stronger uniform-PAC bounds.
This theorem can be proved in the same way as Theorem~5 by \citet{Jiang2016},
which itself is a standard construction involving a careful layering of
difficult instances of the multi-armed bandit problem.\footnote{We here only
use $H/2$ timesteps for bandits and the remaining $H/2$ time steps to accumulate
a reward of $O(H)$ for each bandit} For simplicity, we omitted the dependency
on the failure probability $\delta$, but using the techniques in the proof of
Theorem~26 by \citet{Strehl2009}, a lower bound of order $\Omega\left(
\frac{\numA \numS H^3}{\epsilon^2} \log(SA/\delta) \right)$ can be obtained.
The lower bound shows for small $\epsilon$ the sample complexity of \algname
given in Theorem~\ref{thm:unipacupper} is optimal except for a factor of $H$
and logarithmic terms.

\section{Planning Problem of \algname}
\begin{lem}[Planning Problem]
The policy update
in Lines~\ref{lin:optplan1}--\ref{lin:optplan2} of Algorithm~\ref{alg:fhalg} finds an optimal solution to
the optimization problem
\begin{align}
    \max_{P', V', \pi', r'}  \,\, & \Ex_{s \sim p_0} [V'_{1}(s)]\\
    \forall  s\in \statespace, a \in & \actionspace, t \in [H]: \\
    V'_{H+1} =& 0, \qquad \qquad P'(s,a,t) \in  \Delta_{\numS},  \qquad r'(s,a,t) \in [0,1]\\
    V'_{t}(s) =& r'(s, \pi'(s,t), t) + \Ex_{s' \sim P'(s, \pi'(s,t), t)}[V'_{t+1}] \\
    |(P'(s, a &, t) - \hat P_k(s, a, t))^\top V'_{t+1}| \leq \phi(s,a,t) (H-t)\\ 
    |r'(s, a &, t) - \hat r_k(s, a, t)| \leq \phi(s,a,t)
\end{align}
where $\phi(s,a,t) = \sqrt{\frac{2 \llnp(n(s,a,t)) + \ln(18 \numS \numA H /
\delta)}{n(s,a,t)}}$ is a confidence bound and $\hat P_k(s' | s,a,t) = m(s', s,
a, t) / n(s,a,t)$ are the empirical transition probabilities and $\hat r_k(s,a,t) = l(s,a,t) / n(s,a,t)$ the empirical average rewards.
\label{lem:optplaninterpret}
\end{lem}
\begin{proof}
    Since $\tilde V_{H+1}(\cdot)$ is initialized with $0$ and never changed, we immediately get that it is an optimal value for $V'_{H+1}(\cdot)$ which is constrained to be $0$. Consider now a single time step $t$ and assume $V'_{t+1}$ are fixed to the optimal values $\tilde V_{t+1}$.
    Plugging in the computation of $Q(a)$ into the computation of $\tilde V_{t}(s)$, we get
    \begin{align} 
        \tilde V_t(s) 
        = \max_{a} Q(a) 
        =& \max_{a \in \actionspace} \bigg[ \min\left\{ 1, \hat r(s,a,t) + \phi(s,a,t)\right\}\\
         &+ \min\left\{ \max \tilde V_{t+1}, \indicator{n(s,a,t) > 0}(\hat P(s,a,t)^\top \tilde V_{t+1}) + \phi(s,a,t) (H-t) \right\}\bigg]
    \end{align}
    using the convention that $\hat r(s,a,t) = 0$ if $n(s,a,t) = 0$.
Assuming that $V'_{t+1} = \tilde V_{t+1}$, and that our goal for now is to maximize $\tilde V_t(s)$, this can be rewritten as
    \begin{align} 
        \max_{P'(s,a,t), r'(s,a,t)} \tilde V_t(s)  
        = \max_{P'(s,a,t), r'(s,a,t), \pi'(s, t)} \bigg[& r'(s,\pi'(s,t),t)  + P'(s, \pi'(s, t), t)^\top \tilde V_{t+1}\bigg]\\
          \textrm{s.t.} \qquad \forall a \in \actionspace:  r'(s,a,t) &\in [0,1], \qquad P'(s,a,t) \in \Delta_{\numS}\\
        |(P'(s, a , t) - \hat P_k(s, a, t))^\top V'_{t+1}| \leq & \phi(s,a,t) (H-t)\\ 
         |r'(s, a , t) - \hat r_k(s, a, t)| \leq & \phi(s,a,t)
    \end{align}
    since in this problem either $P'(s,\pi'(s,t),t)^\top \tilde V_{t+1} = \hat
    P(s,\pi'(s,t),t)^\top \tilde V_{t+1} + \phi(s,a,t) (H-t)$ if
    that does not violate $P'(s,\pi'(s,t),t)^\top \tilde V_{t+1} \leq \max
    \tilde V_{t+1}$ and otherwise $P'(s', s, \pi'(s,t), t) = 1$ for one state
    $s'$ with $\tilde V_{t+1}(s') = \max \tilde V_{t+1}$. Similarly, either
    $r'(s,\pi'(s,t),t) = \hat r(s, \pi'(s, t), t) + \phi(s, \pi'(s, t), t)$
    if that does not violate $r'(s, \pi'(s,t), t) \leq 1$ or $r'(s, \pi'(s,t),
    t) = 1$ otherwise.
    Using induction for $t = H, H-1 \dots 1$, we see that \algname computes an optimal solution to 

\begin{align}
    \max_{P', V', \pi', r'}  \,\, & V'_{1}(\tilde s)\\
    \forall  s\in \statespace, a \in & \actionspace, t \in [H]: \\
    V'_{H+1} =& 0, \qquad \qquad P'(s,a,t) \in  \Delta_{\numS},  \qquad r'(s,a,t) \in [0,1]\\
    V'_{t}(s) =& r'(s, \pi'(s,t), t) + \Ex_{s' \sim P'(s, \pi'(s,t), t)}[V'_{t+1}] \\
    |(P'(s, a &, t) - \hat P_k(s, a, t))^\top V'_{t+1}| \leq \phi(s,a,t) (H-t)\\ 
    |r'(s, a &, t) - \hat r_k(s, a, t)| \leq \phi(s,a,t)
\end{align}
for any fixed $\tilde s$. The intersection of all optimal solutions to this problem for all $\tilde s \in \statespace$ are also an optimal solution to 
\begin{align}
    \max_{P', V', \pi', r'}  \,\, & p_0^\top V'_{1}\\
    \forall  s\in \statespace, a \in & \actionspace, t \in [H]: \\
    V'_{H+1} =& 0, \qquad \qquad P'(s,a,t) \in  \Delta_{\numS},  \qquad r'(s,a,t) \in [0,1]\\
    V'_{t}(s) =& r'(s, \pi'(s,t), t) + \Ex_{s' \sim P'(s, \pi'(s,t), t)}[V'_{t+1}] \\
    |(P'(s, a &, t) - \hat P_k(s, a, t))^\top V'_{t+1}| \leq \phi(s,a,t) (H-t)\\ 
    |r'(s, a &, t) - \hat r_k(s, a, t)| \leq \phi(s,a,t).
\end{align}
Hence, \algname computes an optimal solution to this problem.
\end{proof}

\section{Details of PAC Analysis}
In the analysis, we denote the value of $n(\cdot, t)$ after the planning in
iteration $k$ as $n_{tk}(\cdot)$. We further denote by $P(s' | s, a, t)$ the
probability of sampling state $s'$ as $s_{t+1}$ when $s_t = s, a_t = a$. With
slight abuse of notation, $P(s,a,t) \in [0, 1]^{\numS}$ denotes the probability
vector of $P(\cdot | s, a, t)$. We further use $\tilde P_k(s' | s, a, t)$ as
conditional probability of $s_{t+1} = s'$ given $s_t = s, a_t = a$ but in the
optimistic MDP $\tilde M$ computed in the optimistic planning steps in
iteration $k$.
We also use the following definitions:
\begin{align}
    \wmin =& \wminc = \frac{\epsilon c_\epsilon}{H^2 \numS}\\
    c_\epsilon =& \frac 1 3\\
    L_{tk} =& \{ (s,a) \in \saspace \, : \, w_{tk}(s,a) \geq \wmin \}\\
    \llnp(x) =& \ln (\ln (\max\{x, e\}))\\
    \range(x) =& \max(x) - \min(x) \\
    \delta' =& \frac{\delta}{9}
\end{align}
In the following, we provide the formal proof for Theorem~\ref{thm:unipacupper} and then present all necessary lemmas:
\subsection{Proof of Theorem~\ref{thm:unipacupper}}
\begin{proof}[Proof of Theorem~\ref{thm:unipacupper}]
    Corollary~\ref{cor:failprob} ensures that the failure event has probability at most $\delta$.
    Outside the failure event Lemma~\ref{lem:num_nonnice} ensures that all but at most 
    $\frac{48 \numA^2 \numS^3 H^4}{\epsilon} \polylog(\numA, \numS, H, 1 / \epsilon, 1 / \delta)$
    episodes are friendly. Finally, Lemma~\ref{lem:friendlyopt} shows that all friendly episodes except at most  
    $\left(\frac{9216}{\epsilon} + 417 \numS \right) \frac{\numA \numS
    H^4}{\epsilon} \polylog(\numA, \numS, H, 1 / \epsilon, 1 / \delta)$ are
    $\epsilon$-optimal. The second bound follows from replacing $\numA \numS^2$ by $1 / \epsilon$ in the second term.
    Furthermore,
    outside the failure event Lemma~\ref{lem:num_nonnice} ensures that all but at most 
    $\frac{6 \numA \numS^2 H^3}{\epsilon} \polylog(\numA, \numS, H, 1 / \epsilon, 1 / \delta)$
    episodes are nice. Finally, Lemma~\ref{lem:niceopt} shows that all nice episodes except at most  
    $\left(4 + \numS \right)576 \frac{\numA \numS
    H^4}{\epsilon} \polylog(\numA, \numS, H, 1 / \epsilon, 1 / \delta)$ are
    $\epsilon$-optimal.

\end{proof}

\subsection{Failure Events and Their Probabilities}
\label{sec:failureevents}

In this section, we define a failure event $F$ in which we cannot guarantee the performance of \algname. We then show that this event $F$ only occurs with low probability. All our arguments are based on general uniform concentration of measure statements that we prove in Section~\ref{sec:concentrationproofs}. In the following we argue how the apply in our setting and finally combine all concentration results to get $\prob(F) \leq \delta$.
The failure event is defined as
\begin{align}
F =& \bigcup_k \left[ F^N_k \cup F^{CN}_k \cup F^P_k \cup F^V_k  \cup F^{L1}_k \cup F^{R}_k \right]
\end{align}
where
\begin{align}
    F^{N}_k =& \left\{ \exists s, a, t: \, n_{tk}(s, a) < \frac 1 2 \sum_{i < k} w_{ti}(s, a) - \ln \frac{\numS \numA H}{\delta'} \right\}\\
    F^{CN}_k =& \left\{ \exists s, a, s', a', u < t: \, n_{tk}(s, a) < \frac 1 2
     n_{uk}(s', a') \sum_{i < k} w_{ui}^t(s, a | s', a') - \ln \left( \frac{\numS^2 \numA^2 H^2}{\delta'} \right) \right\}\\
    F^{V}_k =& \left\{ \exists s, a, t: \, 
     |(\hat P_k(s, a, t) - P(s, a, t))^\top V^{\star}_{t+1}| 
    \geq 
    \sqrt{\frac{\range(V^{\star}_{t+1})^2}{n_{tk}(s, a)} 
        \left(2 \llnp(n_{tk}(s, a))+ \ln \frac{3 \numS \numA H}{\delta'}\right)
    } \right\}\\
    F^{P}_k =& \bigg\{ \exists s, s', a, t: \, | \hat P_k(s' | s, a, t) - P(s' | s, a, t)| 
    \geq 
    \sqrt{\frac{2 P(s' | s,a ,t)}{n_{tk}(s, a)} 
        \left(2 \llnp(n_{tk}(s, a))+ \ln \frac{3 \numS^2 \numA H}{\delta'}\right)
    }\\
    & \qquad \qquad \qquad +
     \frac{1}{n_{tk}(s, a)} 
     \left(2 \llnp(n_{tk}(s, a))+ \ln \frac{3 \numS^2 \numA H}{\delta'}\right)
    \bigg\}\\
     F^{L1}_k =& \left\{ \exists s, a, t: \, 
    \|\hat P_k(s, a, t) - P(s, a, t)\|_1 
    \geq 
    \sqrt{\frac{4} {n_{tk}(s, a)} 
        \left(2 \llnp(n_{tk}(s, a))+ \ln \frac{3 \numS \numA H(2^\numS - 2)}{\delta'}\right)
    } \right\}\\   
F^{R}_k =& \left\{ \exists s, a, t: \, 
    |\hat r_k(s, a, t) - r(s, a, t)| 
    \geq 
    \sqrt{\frac{1} {n_{tk}(s, a)} 
        \left(2 \llnp(n_{tk}(s, a))+ \ln \frac{3 \numS \numA H}{\delta'}\right)
    } \right\}.
\end{align}
We now bound the probability of each type of failure event individually:

\begin{cor}
    For any $\delta' > 0$, it holds that 
    $\prob\left( \bigcup_{k=1}^\infty F^V_k \right) \leq 2 \delta'$
and $\prob\left( \bigcup_{k=1}^\infty F^R_k \right) \leq 2 \delta'$
    \label{cor:FVbound}
\end{cor}

\begin{proof}
    Consider a fix $s \in \statespace, a \in \actionspace, t \in [H]$ and
    denote $\mathcal F_k$ the sigma-field induced by the first $k-1$
    episodes and the $k$-th episode up to $s_t$ and $a_t$ but not
    $s_{t+1}$. Define $\tau_i$ to be the index of the episode where $(s,a)$
    was observed at time $t$ the $i$th time. Note that $\tau_i$ are
    stopping times with respect to $\mathcal F_i$.  Define now the
    filtration $\mathcal G_i = \mathcal F_{\tau_i} = \{
    A \in \mathcal F_\infty \, : \, A \cap \{\tau_i \leq t \} \in \mathcal F_t \,\,\forall\, t \geq 0\}$ and $X_k = (V^{\star}_{t+1}(s'_{k}) -  P(s, a, t)^\top V^{\star}_{t+1}) \indicator{\tau_k <
    \infty}$ where $s'_{i}$ is the value of $s_{t+1}$ in episode $\tau_i$
    (or arbitrary, if $\tau_i = \infty$). 

    By the Markov property of the MDP, we have that $X_i$ is a martingale
    difference sequence with respect to the filtration $\mathcal G_i$.
    Further, since $\Ex[X_i | \mathcal G_{i-1}] = 0$ and $|X_i| \in [0, \range(V^{\star}_{t+1})]$, $X_i$
    conditionally $\range(V^{\star}_{t+1}) / 2$-subgaussian due to Hoeffding's Lemma, i.e., satisfies $\Ex[\exp(\lambda X_i) |
    \mathcal G_{i-1}] \leq \exp(\lambda^2 \range(V^{\star}_{t+1})^2 / 2)$. 

    We can therefore apply Lemma~\ref{lem:uniformhoeffding} and conclude that
    \begin{align}
        \prob\left( \exists k: |(\hat P_k(s,a,t) - P(s,a,t))^\top V^{\star}_{t+1}|
        \geq 
    \sqrt{\frac{\range(V^{\star}_{t+1})^2}{n_{tk}(s,a)} \left(2 \llnp( n_{tk}(s,a)) + \ln \frac 3 {\delta'} \right) }\right)
        \leq 2\delta'\,.
    \end{align}
    Analogously
    \begin{align}
        \prob\left( \exists k: |\hat r_k(s,a,t) - r(s,a,t)|
        \geq 
    \sqrt{\frac{1}{n_{tk}(s,a)} \left(2 \llnp( n_{tk}(s,a)) + \ln \frac 3 {\delta'} \right) }\right)
        \leq 2\delta'\,.
    \end{align}
    Applying the union bound over all $s \in \statespace, a \in \actionspace$ and $t \in [H]$, we obtain the desired statement for $F^V$. In complete analogy using the same filtration, we can show the statement for $F^R$.
\end{proof}

\begin{cor}
For any $\delta' > 0$, it holds that
    $\prob 
    \left( \bigcup_{k=1}^\infty F^{P}_k \right) \leq 2\delta'$.
   \label{cor:FPbound}
\end{cor}
\begin{proof}
    Consider first a fix $s', s \in \statespace$, $t \in [H]$ and $a \in \actionspace$.
    Let $K$ denote the number of times the triple $s, a, t$ was encountered in total during the run of the algorithm. Define the random sequence $X_i$ as follows. For $i \leq K$, let $X_i$ be the indicator of whether $s'$ was the next state when $s, a, t$ was encountered the $i$th time and for $i > K$, let $X_i \sim \operatorname{Bernoulli}(P(s' | s, a, t))$ be drawn i.i.d. By construction this is a sequence of i.i.d. Bernoulli random variables with mean $P(s' | s, a, t)$. Further the event 
    \begin{align}
        \bigcup_{k} \Bigg\{
            \left|\hat P_k(s' | s,a, t) - P(s' | s, a, t)\right|
            \geq & \sqrt{\frac{2 P(s'|s, a, t)}{n_{tk}(s, a)} \left( 2\llnp(n(s, a, t)) + \ln \frac{3 \numS^2 \numA H}{\delta'}\right)}\\
                 & + \frac 1 {n_{tk}(s, a)} \left(2\llnp(n_{tk}(s, a)) + \ln \frac {3 \numS^2 \numA H}{\delta'} \right)
        \Bigg\}
    \end{align}
    is contained in the event
    \begin{align}
        \bigcup_{i} 
        \left\{
            \left|\hat \mu_i - \mu\right| 
            \geq \sqrt{\frac{2\mu}{i} \left( 2\llnp(i) + \ln \frac{3}{\delta'}\right)} + \frac 1 i \left(2\llnp(i) + \ln \frac {3 \numS^2 \numA H} {\delta'} \right)
        \right\}
    \end{align}
    whose probability can be bounded by $2 \delta' / \numS^2 / \numA / H$ using Lemma~\ref{lem:uniformbern}.
    The statement now follows by applying the union bound.
\end{proof}

\begin{cor}
    For any $\delta' > 0$, it holds that 
    $\prob\left( \bigcup_{k=1}^\infty F^{L1}_k \right) \leq \delta'$
    \label{cor:FL1bound}
\end{cor}
\begin{proof}
Using the same argument as in the proof of Corollary~\ref{cor:FPbound} the statement follows from Lemma~\ref{lem:l1deviation}.
\end{proof}

\begin{cor}
    It holds that
    \begin{align}
        \prob\left( \bigcup_k F^N_k \right) \leq \delta' \quad \textrm{and} \quad
        \prob\left( \bigcup_k F^{CN}_k \right) \leq \delta'.
    \end{align}
    \label{cor:FNbound}
\end{cor}
\begin{proof}
    Consider a fix $s \in \statespace, a \in \actionspace, t \in [H]$.
    We define $\mathcal F_k$ to be the sigma-field induced by the first $k-1$ episodes
    and $X_k$ as the indicator whether $s, a, t$ was observed in episode $k$. The probability $w_{tk}(s, a)$ pf whether $X_k = 1$  is $F_k$ measurable and hence we can apply Lemma~\ref{lem:uniformwnconc} with $W = \ln \frac{\numS \numA H}{\delta'}$ and obtain that 
    $\prob\left( \bigcup_k F^N_k \right) \leq \delta'$ after applying the union bound.

    For the second statement, consider again a fix $s,s' \in \statespace, a,a' \in \actionspace, u,t \in [H]$ with $u < t$ and
    denote by $\mathcal F_k$ the sigma-field induced by the first $k-1$
    episodes and the $k$-th episode up to $s_u$ and $a_u$ but not
    $s_{u+1}$. 
    Define $\tau_i$ to be the index of the episode where $(s',a')$
    was observed at time $u$ the $i$th time. Note that $\tau_i$ are
    stopping times with respect to $\mathcal F_i$.  Define now the
    filtration $\mathcal G_i = \mathcal F_{\tau_i} = \{
        A \in \mathcal F_\infty \, : \, A \cap \{\tau_i \leq k \} \in \mathcal F_k \,\,\forall\, k \geq 0\}$ and $X_i$ to be the indicator whether $s, a, t$ and $s', a', u$ was observed in episode $\tau_i$. If $\tau_i = \infty$, we set $X_i = 0$. Note that the probablity $w_{ui}^t(s, a | s', a') \indicator{\tau_i < \infty}$ of $X_i = 1$ is $\mathcal G_{i}$-measureable.

    By the Markov property of the MDP, we have that $X_i$ is a martingale
    difference sequence with respect to the filtration $\mathcal G_i$.
    We can therefore apply Lemma~\ref{lem:uniformwnconc} with $W = \ln \frac{\numS^2 \numA^2 H^2}{\delta'}$ and using the union bound over all $s, a, s', a', u, t$, we get  
    $
        \prob\left( \bigcup_k F^{CN}_k \right) \leq \delta'$.
\end{proof}

\begin{cor}
    The total failure probability of the algorithm is bounded by
        $\prob\left( F \right) \leq 9 \delta' = \delta$.
    \label{cor:failprob}
\end{cor}
\begin{proof}
    Statement follows directly from Corollary~\ref{cor:FVbound}, Corollary~\ref{cor:FPbound}, Corollary~\ref{cor:FL1bound}, Corollary~\ref{cor:FNbound} and the union bound.
\end{proof}

\subsection{Nice and Friendly Episodes}
\label{sec:nice_and_friendly}
We now define the notion of \emph{nice} and the stronger \emph{friendly}
episodes. In nice episodes, all states either have low probability of occuring
or the sum of probability of occuring in the previous episodes is large enough
so that outside the failure event we can guarantee that
\begin{align}
    n_{tk}(s,a) \geq \frac 1 4 \sum_{i < k} w_{ti}(s,a).
\end{align}
This allows us to then bound the number of nice episodes by the number of times terms of the form
\begin{align}
    \sum_{t=1}^H \sum_{s,a \in L_{tk}} w_{tk}(s,a) \sqrt{\frac{\llnp (n_{tk}(s,a)) + D}{n_{tk}(s,a)}} 
\end{align}
can exceed a chosen threshold (see Lemma~\ref{lem:mainratelemma} below).
In the next section, we will bound the optimality gap of an episode by terms
of such form and use the results derived here to bound the number of nice
episodes where the algorithm can follow a $\epsilon$-suboptimal policy.
Together with a bound on the number of non-nice episodes, we obtain the sample
complexity of \algname shown in Theorem~\ref{thm:unipacupper}.

Similarly, we use a more refined analysis of the optimality gap of friendly
episodes together with Lemma~\ref{lem:ratelemmacond} below to obtain the
tighter sample complexity linear-polylog in $\numS$.

\begin{defn}[Nice and Friendly Episodes]
An episode $k$ is \emph{nice} if and only if for all $s \in \statespace$, $a \in \actionspace$ and $t \in [H]$ the following two conditions hold:
    \begin{align}
        w_{tk}(s,a) \leq \wmin \quad \vee \quad
        \frac 1 4 \sum_{i < k} w_{ti}(s,a) \geq \ln \frac{\numS \numA H}{\delta'}
    \end{align}
An episode $k$ is \emph{friendly} if and only if it is nice and for all $s, s' \in \statespace$, $a, a' \in \actionspace$ and $u,t \in [H]$ with $u < t$ the following two conditions hold:
    \begin{align}
        w_{uk}^t(s,a|s', a') \leq \wminc \quad \vee \quad
        \frac 1 4 \sum_{i < k} w_{ui}^t(s,a|s', a') \geq \ln \frac{\numS^2 \numA^2 H^2}{\delta'}.
    \end{align}
    We denote the set of all nice episodes by $N \subseteq \mathbb N$ and the set of all friendly episodes by $K \subseteq N$.
\end{defn}

\begin{lem}[Properties of nice and friendly episodes]
    If an episode $k$ is nice, i.e., $k \in N$, then on $F^c$ (outside the failure event) for all $s \in \statespace$, $a \in \actionspace$ and $t \in [H]$ with $u < t$ the following statement holds:
    \begin{align}
        w_{tk}(s,a) \leq \wmin \quad \vee \quad
        n_{tk}(s,a) \geq \frac 1 4 \sum_{i < k} w_{ti}(s,a).
    \end{align}
    If an episode $k$ is friendly, i.e., $k \in K$, then on $F^c$ (outside the failure event) for all $s, s' \in \statespace$, $a, a' \in \actionspace$ and $u,t \in [H]$ with $u < t$ the above statement holds as well as
    \begin{align}
        w_{uk}^t(s,a|s', a') \leq \wminc \quad \vee \quad
        n_{tk}(s,a) \geq \frac 1 4 n_{uk}(s', a') \sum_{i < k} w_{ui}^t(s,a|s', a') .
    \end{align}
    \label{lem:nicenessep}
\end{lem}
\begin{proof}
    Since we consider the event ${F_k^N}^c$, it holds for all $s, a, t$ triples with $w_{tk}(s,a) > \wmin$
    \begin{align}
        n_{tk}(s,a) \geq \frac 1 2 \sum_{i < k} w_{ti}(s,a) -  \ln \frac{\numS \numA H}{\delta'}
        \geq \frac 1 4 \sum_{i < k} w_{ti}(s,a)
    \end{align}
    for $k \in N$
    Further, since we only consider the event ${F_k^{CN}}^c$,we have for all $s,s' \in \statespace$, $a, a' \in \actionspace$, $u,t \in [H]$ with $u < t$ and $w_{uk}^t(s,a | s', a') > \wmin$
    \begin{align}
        n_{tk}(s,a) \geq \frac 1 2 n_{uk}(s', a') \sum_{i < k} w_{ui}^t(s,a|s', a') 
        - \ln \frac{\numS^2 \numA^2 H^2}{\delta'}
    \end{align}
for $k \in E$.
    If $n_{uk}(s',a') = 0$ then $n_{tk}(s,a) \geq 0 = \frac 1 4 n_{uk}(s', a') \sum_{i < k} w_{ui}^t(s,a|s', a')$ holds trivially. Otherwise $n_{uk}(s', a') \geq 1$ and therefore
    \begin{align}
        n_{tk}(s,a) \geq & \frac 1 2 n_{uk}(s', a') \sum_{i < k} w_{ui}^t(s,a|s', a') 
        - \ln \frac{\numS^2 \numA^2 H^2}{\delta'}\\
        \geq &
        \frac 1 2 n_{uk}(s', a') \sum_{i < k} w_{ui}^t(s,a|s', a') - \frac 1 4 \sum_{i < k} w_{ui}^t(s,a|s', a')\\
        \geq & \frac 1 4 n_{uk}(s', a') \sum_{i < k} w_{ui}^t(s,a|s', a') 
    \end{align}
\end{proof}

\begin{lem}[Number of non-nice and non-friendly episodes]
    On the good event $F^c$, the number of episodes that are not friendly is at most
    \begin{align}
        &48\frac{\numS^3 \numA^2 H^4}{\epsilon}\ln \frac{\numS^2 \numA^2 H^2}{\delta'}    
    \end{align}
    and the number episodes that are not nice is at most
    \begin{align}
\frac{6\numS^2 \numA H^3}{\epsilon}\ln \frac{\numS \numA H}{\delta'}.
    \end{align}
    \label{lem:num_nonnice}
\end{lem}
\begin{proof}
    If an episode $k$ is not nice, then there is $s,a,t$ with $w_{tk}(s,a) > \wmin$ and $
\sum_{i < k} w_{ti}(s,a) <  4 \ln \frac{\numS \numA H}{\delta'}$.
            Since the sum on the left-hand side of this inequality increases by at least $\wmin$ when this happens and the right hand side stays constant, this situation can occur at most 
            \begin{align}
\frac{4\numS \numA H}{\wmin}\ln \frac{\numS \numA H}{\delta'}
=\frac{24\numS^2 \numA H^3}{\epsilon}\ln \frac{\numS \numA H}{\delta'}
\end{align} 
times in total.
If an episode $k$ is not friendly, it is either not nice or there is $s,a,t$ and $s', a', u$ with $u < t$ and $w^t_{uk}(s',a' | s, a) > \wminc$ and $
\sum_{i < k} w_{ui}^t(s,a| s', a') < 4\ln \frac{\numS^2 \numA^2 H^2}{\delta'}$.
            Since the sum on the left-hand side of this inequality increases by at least $\wminc$ each time this happens while the right hand side stays constant, this can happen at most 
            $\frac{4\numS^2 \numA^2 H^2}{\wminc}\ln \frac{\numS^2 \numA^2 H^2}{\delta'}$ times in total.
    Therefore, there can only be at most
    \begin{align}
        & \frac{4\numS \numA H}{\wmin}\ln \frac{\numS \numA H}{\delta'} + \frac{4\numS^2 \numA^2 H^2}{\wminc}\ln \frac{\numS^2 \numA^2 H^2}{\delta'}\\
        = &
        \frac{4\numS^2 \numA H^3}{c_\epsilon \epsilon}\ln \frac{\numS \numA H}{\delta'} + \frac{4\numS^3 \numA^2 H^4}{c_\epsilon \epsilon}\ln \frac{\numS^2 \numA^2 H^2}{\delta'}
         \leq 
        \frac{48 \numS^3 \numA^2 H^4}{\epsilon^2}\ln \frac{\numS^2 \numA^2 H^2}{\delta'} 
    \end{align}
    non-friendly episodes.
\end{proof}

\begin{lem}[Main Rate Lemma]
Let $r \geq 1$ fix and $C > 0$ which can depend polynomially on the relevant quantities and $\epsilon' > 0$ and let $D \geq 1$ which can depend poly-logarithmically on the relevant quantities. Then
\begin{align}
    \sum_t \sum_{s,a \in L_{tk}}
    w_{tk}(s, a) \left(\frac{C (\llnp(n_{tk}(s,a)) + D)}{n_{tk}(s,a)} \right)^{1/r} \leq \epsilon'
\end{align}
on all but at most 
    \begin{align}
        \frac{8C \numA \numS H^r}{\epsilon'^r} \polylog(\numS, \numA, H, \delta^{-1}, \epsilon'^{-1}).
    \end{align}
    nice episodes.
    \label{lem:mainratelemma}
\end{lem}
\begin{proof}
    Define 
\begin{align}
    \Delta_k =&
    \sum_t \sum_{s,a \in L_{tk}}
    w_{tk}(s, a) \left(\frac{C(\llnp(n_{tk}(s,a))+D)}{n_{tk}(s,a)}\right)^{1/r}\\
    =&
    \sum_t \sum_{s,a \in L_{tk}}
    w_{tk}(s, a)^{1 - \frac 1 r} \left(w_{tk}(s, a)\frac{C(\llnp(n_{tk}(s,a))+D)}{n_{tk}(s,a)}\right)^{1/r}.
\end{align}
    We first bound using H\"older's inequality
    \begin{align}
        \Delta_k \leq 
        \left(\sum_t \sum_{s, a \in L_{tk}}
        \frac{C H^{r-1} w_{tk}(s, a)(\llnp(n_{tk}(s,a))+D)}{n_{tk}(s,a)}\right)^{\frac 1 r}.
    \end{align}
    Using the property in Lemma~\ref{lem:nicenessep} of nice episodes as well as the fact that $w_{tk}(s,a) \leq 1$ and $\sum_{i < k} w_{ti}(s,a) \geq 4 \ln \frac {\numS \numA H}{\delta'} \geq 4 \ln(2) \geq 2$, we bound
    \begin{align}
        n_{tk}(s,a) \geq \frac 1 4 \sum_{i < k} w_{ti}(s,a) \geq \frac 1 8 \sum_{i \leq k} w_{ti}(s,a).
    \end{align}
    The function $\frac{\llnp(x) + D}{x}$ is monotonically decreasing in $x \geq 0$ since $D \geq 1$ (see Lemma~\ref{lem:llnpprop}).
    This allows us to bound
    \begin{align}
        \Delta^r_k \leq &
        \sum_t \sum_{s, a \in L_{tk}}
        \frac{CH^{r-1} w_{tk}(s, a)(\llnp(n_{tk}(s,a))+D)}{n_{tk}(s,a)}\\
        \leq & 8CH^{r-1}
        \sum_t \sum_{s, a \in L_{tk}}
    \frac{w_{tk}(s, a)\left(\llnp\left( \frac 1 8 \sum_{i \leq k} w_{ti}(s,a)\right)+D\right)}{\sum_{i \leq k} w_{ti}(s,a)}\\
    \leq & 8CH^{r-1}
        \sum_t \sum_{s, a \in L_{tk}}
    \frac{w_{tk}(s, a)\left(\llnp\left(\sum_{i \leq k} w_{ti}(s,a)\right)+D\right)}{\sum_{i \leq k} w_{ti}(s,a)}.
    \end{align}
    
    Assume now $\Delta_k > \epsilon'$. In this case the right-hand side of the inequality above is also larger than $\epsilon'^r$ and there is at least one $(s, a, t)$ with $w_{tk}(s,a) > w_{\min}$
    and
    \begin{align}
        \frac{ 8 C \numS \numA H^r\left(\llnp\left(\sum_{i \leq k} w_{ti}(s,a) \right)+D \right)}{\sum_{i \leq k} w_{ti}(s, a)} >& \epsilon'^r \\
        \Leftrightarrow
        \frac{\llnp\left(\sum_{i \leq k} w_{ti}(s,a) \right)+D}{\sum_{i \leq k} w_{ti}(s, a)} >& \frac{\epsilon'^r}{8 C \numS \numA H^r}.
    \end{align}
    Let us denote $C' = \frac{8 C \numA \numS H^r}{\epsilon'^r}$. Since
    $\frac{\llnp(x) + D}{x}$ is monotonically decreasing and $x = C'^2 + 3C'D$
    satisfies $\frac{\llnp(x) + D}{x} \leq \frac{\sqrt{x} + D}{x} \leq \frac 1
    {C'}$, we know that if $\sum_{i \leq k} w_{ti}(s,a) \geq C'^2 + 3C' D$ then
    the above condition cannot be satisfied for $s,a,t$. Since each time the
    condition is satisfied, it holds that $w_{tk}(s,a) > w_{\min}$ and so
    $\sum_{i \leq k} w_{ti}(s,a)$ increases by at least $w_{\min}$, it can happen at most  
    \begin{align}
        m \leq \frac{\numA \numS H (C'^2 + 3C' D)}{w_{\min}} 
    \end{align}
    times that $\Delta_k > \epsilon'$. Define $K = \{ k : \Delta_k > \epsilon' \} \cap N$ and we know that $|K| \leq m$.
        Now we consider the sum
    \begin{align}
        \sum_{k \in K} \Delta_k^r 
        \leq &
        \sum_{k \in K} 
        8CH^{r-1}
        \sum_t \sum_{s, a \in L_{tk}}
        \frac{w_{tk}(s, a) \left(\llnp\left(\sum_{i \leq k} w_{ti}(s,a)\right)+D\right)}{\sum_{i \leq k} w_{ti}(s,a)}\\
        \leq &
        8CH^{r-1}\left(\llnp\left(C'^2 + 3C'D \right)+D\right)
        \sum_t \sum_{s, a \in L_{tk}} 
            \sum_{k \in K}
        \frac{w_{tk}(s, a) }{\sum_{i \leq k} w_{ti}(s,a) \indicator{w_{ti}(s, a) \geq w_{\min}}}
    \end{align}
    For every $(s, a, t)$, we consider the sequence of $w_{ti}(s, a) \in [w_{\min}, 1]$ with $i \in I = \{ i \in \mathbb N \, : \, w_{ti}(s, a) \geq w_{\min}\}$ and apply Lemma~\ref{lem:logseq}.
    This yields that
    \begin{align}
            \sum_{k \in K}
        \frac{w_{tk}(s, a)}{\sum_{i \leq k} w_{ti}(s,a) \indicator{w_{ti}(s, a) \geq w_{\min}}}
        \leq 
        1 + \ln(m / w_{\min}) = \ln\left(\frac{me}{w_{\min}}\right)
    \end{align}
    and hence
    \begin{align}
        \sum_{k \in K} \Delta_k^r 
        \leq &
        8C\numA \numS H^r \ln\left(\frac{me}{w_{\min}}\right)\left(\llnp\left(C'^2 + 3C'D \right)+D\right)
    \end{align}
    Since each element in $K$ has to contribute at least $\epsilon'^r$ to this bound, we can conclude that
    \begin{align}
        \sum_{k \in N} \indicator{\Delta_k \geq \epsilon'} \leq \sum_{k \in K} \indicator{\Delta_k \geq \epsilon'} \leq |K| \leq  
        \frac{8C\numA \numS H^r}{\epsilon'^r} \ln\left(\frac{me}{w_{\min}}\right)\left(\llnp\left(C'^2 + 3C'D \right)+D\right).
    \end{align}
    Since $\ln\left(\frac{me}{w_{\min}}\right)\left(\llnp\left(C'^2 + 3C'D \right)+D\right)$ is $\polylog(\numS, \numA, H, \delta^{-1}, \epsilon'^{-1})$, the proof is complete.
\end{proof}

\begin{lem}[Conditional Rate Lemma]
    Let $r \geq 1$ fix and $C > 0$ which can depend polynomially on the
    relevant quantities and $\epsilon' > 0$ and let $D \geq 1$ which can depend
    poly-logarithmically on the relevant quantities. Further $T \subset [H]$ is
    a subset of time-indices with $u < t$ for all $t \in T$. Then
\begin{align}
    \sum_{t \in T} \sum_{s,a \in L_{k}^{ut}}
    w_{uk}^t(s, a | s', a') \left(\frac{C (\llnp(n_{tk}(s,a)) + D)}{n_{tk}(s,a)} \right)^{1/r} 
    \leq \epsilon'\left( \frac{\llnp(n_{uk}(s',a') + D + 1}{n_{uk}(s',a')} \right)^{1/r} 
\end{align}
on all but at most 
    \begin{align}
        \frac{8C \numA \numS |T|^r}{\epsilon'^r} \polylog(\numS, \numA, H, \delta^{-1}, \epsilon'^{-1}).
    \end{align}
    friendly episodes $E$.
    \label{lem:ratelemmacond}
\end{lem}
\begin{proof}
    The proof follows mainly the structure of Lemma~\ref{lem:mainratelemma}. For the sake of completeness, we still present all steps here.
    Define 
\begin{align}
    \Delta_k =&
    \sum_{t \in T} \sum_{s,a \in L_{k}^{ut}}
    w_{uk}^t(s, a | s', a') \left(\frac{C (\llnp(n_{tk}(s,a)) + D)}{n_{tk}(s,a)} \right)^{1/r} \\
    =&
    \sum_{t \in T} \sum_{s,a \in L_{k}^{ut}}
    w_{uk}^t(s, a | s', a')^{1 - 1 / r} \left(w_{uk}^t(s, a | s', a')\frac{C (\llnp(n_{tk}(s,a)) + D)}{n_{tk}(s,a)} \right)^{1/r}.
\end{align}
    We first bound using H\"older's inequality
    \begin{align}
        \Delta_k \leq \left(
    \sum_{t \geq u} \sum_{s,a \in L_{k}^{ut}}
w_{uk}^t(s, a | s', a') \frac{C |T|^{r-1} (\llnp(n_{tk}(s,a)) + D)}{n_{tk}(s,a)} \right)^{\frac 1 r}
    \end{align}
    Using the property in Lemma~\ref{lem:nicenessep} of friendly episodes as well as the fact that $w_{uk}^t(s,a|s', a') \leq 1$ and $\sum_{i < k} w_{ui}^t(s,a|s', a') \geq 4 \ln \frac {\numS^2 \numA^2 H^2}{\delta'} \geq 4 \ln(2) \geq 2$, we bound
    \begin{align}
        n_{tk}(s,a) 
        \geq \frac 1 4 n_{uk}(s',a')\sum_{i < k} w_{ui}^t(s,a|s',a') 
        \geq \frac 1 8 n_{uk}(s',a') \sum_{i \leq k} w_{ui}^t(s,a|s', a').
    \end{align}
    The function $\frac{\llnp(x) + D}{x}$ is monotonically decreasing in $x \geq 0$ since $D \geq 1$ (see Lemma~\ref{lem:llnpprop}).
    This allows us to bound
    \begin{align}
        \Delta^r_k 
        \leq &
        \sum_{t \in T} \sum_{s,a \in L_{k}^{ut}}
        w_{uk}^t(s, a | s', a') \frac{C|T|^{r-1} (\llnp(n_{tk}(s,a)) + D)}{n_{tk}(s,a)} \\
        \leq & 
        8 C|T|^{r-1}\sum_{t \in T} \sum_{s,a \in L_{k}^{ut}}
         \frac{w_{uk}^t(s, a | s', a')(\llnp\left(\frac 1 8 n_{uk}(s',a') \sum_{i \leq k} w_{ui}^t(s,a|s', a')\right) + D)}
         {n_{uk}(s',a') \sum_{i \leq k} w_{ui}^t(s,a|s', a')} \\
        \leq & 
        8 C|T|^{r-1}\sum_{t \in T} \sum_{s,a \in L_{k}^{ut}}
        \frac{w_{uk}^t(s, a | s', a')(\llnp\left(\sum_{i \leq k} w_{ui}^t(s,a|s', a')\right) + \llnp(n_{uk}(s',a')) + D + 1)}
         {n_{uk}(s',a') \sum_{i \leq k} w_{ui}^t(s,a|s', a')},
    \end{align}
    where for the last line we used the first and last property in Lemma~\ref{lem:llnpprop}. For notational convenience, we will use $D' = D + 1 + \llnp(n_{uk}(s', a'))$.
    Assume now $\Delta_k > \epsilon' \left( \frac{D'}{n_{uk}(s',a')} \right)^{1/r}$. In this case the right-hand side of the inequality above is also larger than $\epsilon'^r  \left( \frac{D'}{n_{uk}(s',a')} \right)$ and there is at least one $(s, a, t)$ with $w_{uk}^t(s,a|s', a') > w_{\min}$
    and
    \begin{align}
        \frac{ 8 C \numS \numA |T|^r
        \left(\llnp\left(\sum_{i \leq k} w_{ui}^t(s,a|s',a') \right)+D' \right)}{\sum_{i \leq k} w_{ui}^t(s, a|s', a')} >& D' \epsilon'^r \\
        \Leftrightarrow
        \frac{ 
        \left(\llnp\left(\sum_{i \leq k} w_{ui}^t(s,a|s',a') \right)+D' \right)}{\sum_{i \leq k} w_{ui}^t(s, a|s', a')} >& \frac{D' \epsilon'^r}{8 C \numS \numA |T|^r}.
    \end{align}
    Let us denote $C' = \frac{8 C \numA \numS |T|^r}{\epsilon'^r}$. Since
    $\frac{\llnp(x) + D'}{x}$ is monotonically decreasing and $x = C'^2 + 3 C'$
    satisfies $\frac{\llnp(x) + D'}{x} \leq \frac{\sqrt{x} + D'}{x} \leq D'\frac{\sqrt{x} + 1}{x} \leq \frac {D'}
    {C'}$, we know that if $\sum_{i \leq k} w_{ui}^t(s,a|s', a') \geq C'^2 + 3C'$ then
    the above condition cannot be satisfied for $s,a,t$. Since each time the
    condition is satisfied, it holds that $w_{uk}^t(s,a|s', a') > w_{\min}$ and so
    $\sum_{i \leq k} w_{ui}^t(s,a|s', a')$ increases by at least $w_{\min}$, it can happen at most  
    \begin{align}
        m \leq \frac{\numA \numS |T| (C'^2 + 3C')}{w_{\min}} 
    \end{align}
    times that $\Delta_k > \epsilon'\left( \frac{D'}{n_{uk}(s',a')} \right)^{1/r}$. Define $K = \left\{ k : \Delta_k > \epsilon'\left( \frac{D'}{n_{uk}(s',a')} \right)^{1/r} \right\} \cap E$ and we know that $|K| \leq m$.
        Now we consider the sum
    \begin{align}
        \sum_{k \in K} \Delta_k^r 
        \leq &
        \sum_{k \in K}
        8 C|T|^{r-1}\sum_{t \in T} \sum_{s,a \in L_{k}^{ut}}
        \frac{w_{uk}^t(s, a | s', a')(\llnp\left(\sum_{i \leq k} w_{ui}^t(s,a|s', a')\right) + D')}
         {n_{uk}(s',a') \sum_{i \leq k} w_{ui}^t(s,a|s', a')}\\
        \leq &
         \frac{8 C|T|^{r-1}(\llnp\left(C'^2 + 3C'\right) + D')}{n_{uk}(s',a')}
        \sum_{t \in T} \sum_{s,a \in L_{k}^{ut}}\sum_{k \in K}
        \frac{w_{uk}^t(s, a | s', a')}
         { \sum_{i \leq k} w_{ui}^t(s,a|s', a')}\\
        \leq &
         \frac{8 C|T|^{r-1}D'(\llnp\left(C'^2 + 3C'\right) + 1)}{n_{uk}(s',a')}
        \sum_{t \in T} \sum_{s,a \in L_{k}^{ut}}\sum_{k \in K}
        \frac{w_{uk}^t(s, a | s', a')}
        { \sum_{i \leq k} w_{ui}^t(s,a|s', a') \indicator{w_{ui}^t(s,a|s',a') \geq w_{\min}}}
    \end{align}
    For every $(s, a, t)$, we consider the sequence of $w_{ui}^t(s, a|s',a')
    \in [w_{\min}, 1]$ with $i \in I = \{ i \in \mathbb N \, : \, w_{ui}^t(s, a|s',a')
    \geq w_{\min}\}$ and apply Lemma~\ref{lem:logseq}.
    This yields that
    \begin{align}
        \sum_{k \in K}
        \frac{w_{uk}^t(s, a | s', a')}
        { \sum_{i \leq k} w_{ui}^t(s,a|s', a') \indicator{w_{ui}^t(s,a|s',a') \geq w_{\min}}}
        \leq 
        \ln\left(\frac{me}{w_{\min}}\right)
    \end{align}
    and hence
    \begin{align}
        \sum_{k \in K} \Delta_k^r 
        \leq &
        \frac{8 C\numA \numS |T|^{r}D'(\llnp\left(C'^2 + 3C'\right) + 1)}{n_{uk}(s',a')}  \ln\left(\frac{me}{w_{\min}}\right)
    \end{align}
    Since each element in $K$ has to contribute at least $\frac{D'\epsilon'^r}{n_{uk}(s',a')}$ to this bound, we can conclude that
    \begin{align}
        \sum_{k \in E} \indicator{\Delta_k \geq \epsilon'} =& \sum_{k \in K} \indicator{\Delta_k \geq \epsilon'}\\ 
        \leq&  |K| \leq  
        \frac{8C\numA \numS |T|^r}{\epsilon'^r} \ln\left(\frac{me}{w_{\min}}\right)\left(\llnp\left(C'^2 + 3C' \right)+1\right).
    \end{align}
    Since $ \ln\left(\frac{me}{w_{\min}}\right)\left(\llnp\left(C'^2 + 3C' \right)+1\right) $ is $\polylog(\numS, \numA, H, \delta^{-1}, \epsilon'^{-1})$, the proof is complete.
\end{proof}

\begin{lem}
    Let $a_i$ be a sequence taking values in $[a_{\min}, 1]$ with $a_{\min} > 0$ and $m > 0$, then
    \begin{align}
        \sum_{k=1}^m \frac{a_k}{\sum_{i=1}^k a_i} \leq \ln\left( \frac{me}{a_{\min}}\right).
    \end{align}
    \label{lem:logseq}
\end{lem}
\begin{proof}
    Let $f$ be a step-function taking value $a_i$ on $[i-1, i)$ for all $i$. We have $F(t) \defeq \int_{0}^t f(x) dx = \sum_{i=1}^t a_i$. 
    By the fundamental theorem of Calculus, we can bound
    \begin{align}
\sum_{k=1}^m \frac{a_k}{\sum_{i=1}^k a_i}
        =& \frac{a_1}{a_1} + \int_1^m \frac{f(x)}{F(x) - F(0)} dx 
        = 1 + \ln F(m) - \ln F(1)\\
        \leq & 1 + \ln(m) - \ln a_{\min}
        = \ln\left( \frac{m e}{a_{\min}} \right),
    \end{align}
    where the inequality follows from $a_1 \geq a_{\min}$ and $\sum_{i=1}^m a_i \leq m$.
\end{proof}

\begin{lem}[Properties of $\llnp$]
    The following properties hold:
    \begin{enumerate}
        \item $\llnp$ is continuous and nondecreasing.
        \item $f(x) = \frac{\llnp(nx) + D}{x}$ with $n \geq 0$ and $D \geq 1$ is monotonically decreasing on $\reals_+$.
        \item $\llnp(x y) \leq \llnp(x) + \llnp(y) + 1$ for all $x,y \geq 0$.
    \end{enumerate}
    \label{lem:llnpprop}
\end{lem}
\begin{proof}
    \begin{enumerate}
        \item For $x \leq e$ we have $\llnp(x) = 0$ and for $x \geq e$ we have $\llnp(x) = \ln(\ln(x))$ which is continuous and monotonically increasing and $\lim_{x \searrow e}\ln(\ln(x)) = 0$.
        \item
            The function $\llnp$ is continuous as well as $1/x$ on $\reals_+$ and therefore so it $f$. Further,
    $f$ is differentiable except at $x = e/n$. For $x \in [0, e/n)$, we have $f(x)
        = D / x$ with derivative $- D / x^2 < 0$. Hence $f$ is monotonically
        decreasing on $x \in [0, e/n)$. 
    For $x > e/n$, we have $f(x) = \frac{\ln(\ln(nx)) + D}{x}$ with derivative
    \begin{align}
        -\frac{D + \ln(\ln(nx))}{x^2} + \frac{1}{x^2 \ln(nx)}
        = 
        \frac{1 - \ln(nx) (D + \ln(\ln(nx)))}{x^2 \ln(nx)}.
    \end{align}
    The denominator is always positive in this range so $f$ is monotonically decreasing if and only if
    $\ln(nx)(D - \ln(\ln(nx))) \geq 1$. Using $D \geq 1$, we have $\ln(nx)(D + \ln(\ln(nx))) \geq 1(1 + 0) = 1$. 

\item
    First note that for $xy \leq e^e$ we have $\llnp(xy) \leq 1 \leq \llnp(x) +
    \llnp(y) + 1$ and therfore the statement holds for $x,y \leq e$.

    Then consider the case that $x,y \geq e$ and $\llnp(x) +
    \llnp(y) + 1 - \llnp(xy) =\ln\ln x + \ln \ln y + 1 - \ln(\ln(x) + \ln(y)) =
    -\ln(a+b) + 1 + \ln(a) + \ln(b)$ where $a = \ln x \geq 1$ and $b = \ln y \geq
    1$. The function $g(a,b) = -\ln(a+b) + 1 + \ln(a) + \ln(b)$ is continuous and
    differentiable with $\frac {\partial g }{\partial a} = \frac{b}{a(a+b)} >
    0$ and $\frac {\partial g }{\partial b} = \frac{a}{b(a+b)}
    > 0$. Therefore, $g$ attains its minimum on $[1, \infty) \times [1, \infty)$ at $a=1, b=1$. Since $g(1,1) = 1 - \ln(2) \geq 0$, the statement also holds for $x,y \geq e$.

    Finally consider the case where $x \leq e \leq y$. Then $\llnp(xy) \leq
    \llnp(ey) = \ln(1 + \ln y) \leq \ln \ln y + 1 \leq \llnp(x) + \llnp(y) +
    1$.  Due to symmetry this also holds for $y \leq e \leq x$.
    \end{enumerate}
\end{proof}

\subsection{Decomposition of Optimality Gap}
In this section we decompose the optimality gap and then bound each term
individually. Finally, both rate lemmas presented in the previous section are
used to determine a bound on the number of nice / friendly episodes where the
optimality gap can be larger than $\epsilon$. 
The decomposition in the following lemma is a the simpler version bounding the
number of $\epsilon$-suboptimal nice episodes and eventually lead to the first
bound in Theorem~\ref{thm:unipacupper}.
\begin{lem}[Optimality Gap Bound On Nice Episodes]
 On the good event $F^c$ it holds that  
 $V^{\star}_1(s_0) - V^{\pi_k}_1(s_0) \leq \epsilon$ 
on all nice episodes $k \in N$ except at most
\begin{align}
\frac{144 (4 + 3H^2 + 4\numS H^2) \numA \numS H^2}{\epsilon^2}\polylog(\numA, \numS, H, 1/\epsilon, 1 / \delta)
\end{align}
episodes.
    \label{lem:niceopt}
\end{lem}
\begin{proof}
    Using optimism of the algorithm shown in Lemma~\ref{lem:optimism}, we can bound
\begin{align}
    &V^{\star}_1(s_0) - V^{\pi_k}_1(s_0) \\
    \leq & |\tilde V^{\pi_k}_1(s_0) - V^{\pi_k}_1(s_0)| \label{eqn:opt1}\\
    \leq & \sum_{t=1}^H \sum_{s, a} w_{tk}(s, a) |(\tilde P_k(s, a, t) - P(s, a, t))^\top \tilde V^{\pi_k}_{t+1}|
            + \sum_{t=1}^H \sum_{s, a} w_{tk}(s, a) |\tilde r_k(s, a, t) - r(s, a, t)|\\
    \leq & \sum_{t=1}^H \sum_{s, a \in L_{tk}} w_{tk}(s, a) |(\tilde P_k(s, a, t) - P(s, a, t))^\top \tilde V^{\pi_k}_{t+1}|
+\sum_{t=1}^H \sum_{s, a \in L_{tk}} w_{tk}(s, a) |\tilde r_k(s, a, t) - r(s, a, t)|
\\
       & + \sum_{t=1}^H \sum_{s, a \notin L_{tk}} w_{tk}(s, a) |(\tilde P_k(s, a, t) - P(s, a, t))^\top \tilde V^{\pi_k}_{t+1}|
       + \sum_{t=1}^H \sum_{s, a \notin L_{tk}} w_{tk}(s, a) |\tilde r_k(s, a, t) - r(s, a, t)|\\
    \leq & 
    \sum_{t=1}^H \sum_{s, a \notin L_{tk}} H \wmin +
    \sum_{t=1}^H \sum_{s, a \in L_{tk}} w_{tk}(s, a) \bigg[ |(\tilde P_k(s, a, t) - \hat P_k(s, a, t))^\top \tilde V^{\pi_k}_{t+1}| \\
        & + |(\hat P_k(s, a, t) - P(s, a, t))^\top \tilde V^{\pi_k}_{t+1}|
       +  |\tilde r_k(s, a, t) - r(s, a, t)|\bigg]\label{eqn:decomp1}
\end{align}
    The first term is bounded by $c_\epsilon \epsilon = \frac{\epsilon}{3}$.
    We now can use 
    Lemma~\ref{lem:tildepdiffV}, Lemma~\ref{lem:pdiffVsimple} to bound the other terms by
\begin{align}
\sum_{t=1}^H \sum_{s, a \in L_{tk}} w_{tk}(s, a)
\sqrt{\frac{8(H + H \sqrt{\numS} + 2)^2}{n_{tk}(s, a)} \left(\llnp(n_{tk}(s,a)) + \frac 1 2 \ln \frac{6 \numS \numA H} \delta' \right)}.
\end{align}
We can then apply Lemma~\ref{lem:mainratelemma} with $r=2$, $C = 8 (H + H
\sqrt{\numS}+2)^2$, $D = \frac 1 2\ln \frac{6 \numS \numA H}{\delta'}$ ($\geq 1$ for any
nontrivial setting) and $\epsilon' = 2\epsilon / 3$ to bound this term by
$\frac{2\epsilon}{3}$ on all nice episodes except at most
\begin{align}
    &\frac{64 (H + \sqrt{\numS} H + 2)^2 \numA \numS H^2 3^2}{4\epsilon^2}\polylog(\numA, \numS, H, 1/\epsilon, 1 / \delta)\\
    \leq & \frac{144 (4 + 3H^2 + 4\numS H^2) \numA \numS H^2}{\epsilon^2}\polylog(\numA, \numS, H, 1/\epsilon, 1 / \delta)
\end{align}

Hence $V^{\star}_1(s_0) - V_1^{\pi_k}(s_0) \leq \epsilon$ holds on all nice episodes except those.
\end{proof}
The lemma below is a refined version of the bound above and uses the stronger concept of friendly episodes 
to eventually lead to the second bound in Theorem~\ref{thm:unipacupper}.
\begin{lem}[Optimality Gap Bound On Friendly Episodes]
    On the good event $F^c$ it holds that $p_0^\top (V_1^{\star} - V_1^{\pi_k}) \leq \epsilon$ 
    on all friendly episodes $E$ except at most
    \begin{align}
        \left(\frac{9216}{\epsilon} + 417 \numS\right)\frac{\numA \numS H^4}{\epsilon} \polylog(\numS, \numA, H, 1/\epsilon, \delta)
    \end{align}
    episodes if $\delta' \leq \frac{3 \numA \numS^2 H}{e^2}$.
    \label{lem:friendlyopt}
\end{lem}
\begin{proof}
We can further decompose the optimality gap bound in Equation~\eqref{eqn:decomp1} in the proof of Lemma~\ref{lem:niceopt} as
\begin{align}
    &\sum_{t=1}^H \sum_{s, a \notin L_{tk}} (H+1) \wmin +
    \sum_{t=1}^H \sum_{s, a \in L_{tk}} w_{tk}(s, a) \bigg[ |(\tilde P_k(s, a, t) - \hat P_k(s, a, t))^\top \tilde V^{\pi_k}_{t+1}| 
 +  |\tilde r_k(s, a, t) - r(s, a, t)|      \\
 & + |(\hat P_k(s, a, t) - P(s, a, t))^\top V^{\star}_{t+1}|
 + |(\hat P_k(s, a, t) - P(s, a, t))^\top (V^{\star}_{t+1} -\tilde V^{\pi_k}_{t+1}) |
       \bigg].\\
       \leq
    & c_\epsilon \epsilon +
    \sum_{t=1}^H \sum_{s, a \in L_{tk}} w_{tk}(s, a) \bigg[ |(\tilde P_k(s, a, t) - \hat P_k(s, a, t))^\top \tilde V^{\pi_k}_{t+1}| 
    +  |\tilde r_k(s, a, t) - r(s, a, t)| \\
+& \qquad \qquad |(\hat P_k(s, a, t) - P(s, a, t))^\top V^{\star}_{t+1}|\bigg]\\
 & + \sum_{t=1}^H \sum_{s, a \in L_{tk}} w_{tk}(s, a) |(\hat P_k(s, a, t) - P(s, a, t))^\top (V^{\star}_{t+1} -\tilde V^{\pi_k}_{t+1}) |.
    \label{eqn:lasd}
\end{align}
The second term can be bounded using Lemmas~\ref{lem:pdiffV}, \ref{lem:pdiffVsimple} and \ref{lem:tildepdiffV} by
\begin{align}
\sum_{t=1}^H \sum_{s, a \in L_{tk}} w_{tk}(s, a)
\sqrt{\frac{32(H + 1)^2}{n_{tk}(s, a)} \left(\llnp(n_{tk}(s,a)) + \frac 1 2 \ln \frac{6 \numS \numA H}{\delta'} \right)}.
\end{align}
which we bound by $\epsilon / 3$  using Lemma~\ref{lem:mainratelemma} with $r=2$, $C = 32(H+1)^2$, $D = \frac 1 2 \ln \frac{6 \numS \numA H}{\delta'}$ and $\epsilon' = \epsilon / 3$ on all friendly episodes except at most
\begin{align}
    \frac{8C\numA \numS H^2}{\epsilon'^2} \polylog(\numS, \numA, H, 1/\epsilon, 1/\delta) 
    \leq \frac{9216 \numA \numS H^4}  \polylog(\numS, \numA, H, 1/\epsilon, 1/\delta).
\end{align}
Finally, we apply Lemma~\ref{lem:lowerorder} bound to bound  the last term in Equation~\ref{eqn:lasd} by
    $\epsilon / 3$ on all friendly epsiodes but at most
    \begin{align}
    \frac{417 \numA \numS^2 H^4}{\epsilon}  \polylog(\numS, \numA, H, 1/\delta, 1/\epsilon).
    \end{align}
    It hence follows that $p_0^\top (V_1^{\star} - V_1^{\pi_k}) \leq \epsilon$ on all friendly episodes but at most
    \begin{align}
        \left(\frac{9216 \numA \numS H^4}{\epsilon^2} + \frac{417 \numA \numS^2 H^4}{\epsilon}\right)  \polylog(\numS, \numA, H, 1/\delta, 1/\epsilon).
    \end{align}

\end{proof}
\begin{lem}[Algorithm Learns Fast Enough]
    It holds for all $s \in \statespace, a \in \actionspace$ and $t \in [H]$ 
    \begin{align}
     |(\hat P_k(s, a, t) - \tilde P_k(s, a, t))^\top \tilde V_{t+1}|
    \leq &
     \sqrt{\frac{2H^2}{n_{tk}(s, a)} \left(\llnp(n_{tk}(s,a)) + \frac 1 2 \ln \frac{3 \numS \numA H}{\delta'} \right)}.
    \end{align}
    \label{lem:tildepdiffV}
\end{lem}
\begin{proof}
    Using the definition of the constraint in the planning step of the algorithm shown in Lemma~\ref{lem:optplaninterpret} we can bound
\begin{align}
    |(\hat P_k(s, a, t) - \tilde P_k(s, a, t))^\top \tilde V_{t+1}|
    \leq &
    \sqrt{\frac{H^2}{n_{tk}(s, a)} \left(2\llnp(n_{tk}(s,a)) + \ln \frac{3 \numS \numA H} {\delta'} \right)}.\\
    \leq &
    \sqrt{\frac{2H^2}{n_{tk}(s, a)} \left(\llnp(n_{tk}(s,a)) + \frac 1 2 \ln \frac{3 \numS \numA H}{\delta'} \right)}.
\end{align}
\end{proof}

\begin{lem}[Basic Decompsition Bound]
    On the good event $F^c$ it holds for all $s \in \statespace, a \in \actionspace$ and $t \in [H]$ 
    \begin{align}
    |(\hat P_k(s, a, t) - P(s, a, t))^\top \tilde V_{t+1}|
\leq & \sqrt{\frac{8H^2 \numS} {n_{tk}(s, a)} 
        \left( \llnp(n_{tk}(s, a))+ \frac 1 2\ln \frac{6 \numS \numA H}{\delta'}\right)}
\\
    |\tilde r_k(s, a, t) - r(s, a, t)|
\leq & \sqrt{\frac{4} {n_{tk}(s, a)} 
        \left(\llnp(n_{tk}(s, a))+ \frac 1 2\ln \frac{3 \numS \numA H}{\delta'}\right)}.
    \end{align}
    \label{lem:pdiffVsimple}
\end{lem}
\begin{proof}
On the good event $(F^{L1}_k)c$ we have using H\"older's inequality
\begin{align}
    |(\hat P_k(s, a, t) - P(s, a, t))^\top \tilde V_{t+1}|
\leq & \|\hat P_k(s, a, t) - P(s, a, t))\|_1 \|\tilde V_{t+1}\|_\infty\\
\leq & H \sqrt{\frac{4} {n_{tk}(s, a)} 
        \left(2 \llnp(n_{tk}(s, a))+ \ln \frac{3 \numS \numA H(2^\numS - 2)}{\delta'}\right)}\\
\leq & \sqrt{\frac{8H^2 \numS} {n_{tk}(s, a)} 
        \left( \llnp(n_{tk}(s, a))+ \frac 1 2\ln \frac{6 \numS \numA H}{\delta'}\right)}.
\end{align}
Further, on $(F^R_k)^c$ we have
\begin{align}
    |\tilde r_k(s, a, t) - r(s, a, t)|
    \leq & |\tilde r_k(s, a, t) - r(s, a, t)| 
+ |\tilde r_k(s, a, t) - \hat r(s, a, t)|\\
\leq & 2\sqrt{\frac{1} {n_{tk}(s, a)} 
        \left(2 \llnp(n_{tk}(s, a))+ \ln \frac{3 \numS \numA H}{\delta'}\right)}
\end{align}
\end{proof}

\begin{lem}[Fixed V Term Confidence Bound]
    On the good event $F^c$ it holds for all $s \in \statespace, a \in \actionspace$ and $t \in [H]$ 
    \begin{align}
        |(\hat P_k(s, a, t) - P(s, a, t))^\top V^{\star}_{t+1}|
    \leq &
    \sqrt{\frac{2H^2}{n_{tk}(s, a)} \left(\llnp n_{tk}(s,a) + \frac 1 2 \ln \frac {3 \numS \numA H} {\delta'}\right)}
    \end{align}
    \label{lem:pdiffV}
\end{lem}
\begin{proof}
    Since we consider the event $(F^{V}_k)^c$, we can bound
\begin{align}
    |(\hat P_k(s, a, t) - P(s, a, t))^\top V^{\star}_{t+1}|
    \leq &
    \sqrt{\frac{2H^2}{n_{tk}(s, a)} \left(\llnp n_{tk}(s,a) + \frac 1 2 \ln \frac {3 \numS \numA H} {\delta'}\right)}
\end{align}
\end{proof}

\begin{lem}[Lower Order Term]
    Assume $\delta' \leq \frac{3 \numA \numS^2 H}{e^2}$. On the good event $F^c$ on all friendly episodes $k \in E$ except at most
        $\frac{417 \numA \numS^2 H^4}{\epsilon}  \polylog(\numS, \numA, H, 1/\delta, 1/\epsilon).
$ it holds that
    \begin{align}
        \sum_{t=1}^H \sum_{s, a \in L_{tk}} w_{tk}(s, a) |(\hat P_k(s, a, t) - P(s, a, t))^\top (\tilde V^{\pi_k}_{t+1} - V^{\star}_{t+1})| \leq \frac \epsilon 3.
    \end{align}
    \label{lem:lowerorder}
\end{lem}
\begin{proof}
\begin{align}
    & \sum_{t=1}^H \sum_{s, a \in L_{tk}} w_{tk}(s, a) |(\hat P_k(s, a, t) - P(s, a, t))^\top (\tilde V^{\pi_k}_{t+1} - V^{\star}_{t+1})|\\
    \leq & 
    \sum_{t=1}^H \sum_{s, a \in L_{tk}} w_{tk}(s, a)\sum_{s'}  \sqrt{\frac{2P(s' | s, a, t)}{n_{tk}(s, a)}\left(2 \llnp(n_{tk}(s,a)) + \ln \frac{3\numS^2 \numA H}{\delta'}\right)} |\tilde V^{\pi_k}_{t+1}(s') - V^{\star}_{t+1}(s')|\\
    & + \sum_{t=1}^H \sum_{s, a \in L_{tk}} w_{tk}(s, a)\sum_{s'}  \frac{1}{n_{tk}(s, a)}\left(2 \llnp(n_{tk}(s,a)) + \ln \frac{3\numS^2 \numA H}{\delta'}\right) |\tilde V^{\pi_k}_{t+1}(s') - V^{\star}_{t+1}(s')|\\
    \leq & 
    \sum_{t=1}^H \sum_{s, a \in L_{tk}} w_{tk}(s, a)\sum_{s'}  \sqrt{\frac{2P(s' | s, a, t)}{n_{tk}(s, a)}\left(2 \llnp(n_{tk}(s,a)) + \ln \frac{3\numS^2 \numA H}{\delta'}\right)\left( \tilde V^{\pi_k}_{t+1}(s') - V^{\star}_{t+1}(s')\right)^2 } \\
    & + \sum_{t=1}^H \sum_{s, a \in L_{tk}} \frac{w_{tk}(s, a) H \numS}{n_{tk}(s, a)}\left(2 \llnp(n_{tk}(s,a)) + \ln \frac{3\numS^2 \numA H}{\delta'}\right)\\
    \leq & 
    \sum_{t=1}^H \sum_{s, a \in L_{tk}} w_{tk}(s, a)\sqrt{\frac{2\numS}{n_{tk}(s, a)}  \left(2 \llnp(n_{tk}(s,a)) + \ln \frac{3\numS^2 \numA H}{\delta'}\right)P(s, a, t)^\top\left( \tilde V^{\pi_k}_{t+1} - V^{\star}_{t+1}\right)^2 } \\
    & + \sum_{t=1}^H \sum_{s, a \in L_{tk}} \frac{w_{tk}(s, a) H \numS}{n_{tk}(s, a)}\left(2 \llnp(n_{tk}(s,a)) + \ln \frac{3\numS^2 \numA H}{\delta'}\right)
\end{align}
The first inequality follows since we only consider outcomes in the event
$(F_k^P)^c$, the second from the fact that value function are in the range $[0,
H]$ and the third is an application of the Cauchy-Schwarz inequality.  Using of
optimism of the algorithm (Lemma~\ref{lem:optimism}), we now bound
$P(s,a,t)^\top \left( \tilde V^{\pi_k}_{t+1} - V^{\star}_{t+1}\right)^2 \leq
P(s,a,t)^\top \left( \tilde V^{\pi_k}_{t+1} - V^{\pi_k}_{t+1}\right)^2$ which
we bound by $c_\epsilon \epsilon + \left( c_\epsilon \epsilon +
\sqrt{\frac{C'^2}{n_{tk}(s,a) \numS} \left( \llnp(n_{tk}(s,a) + \frac 1 2 \ln
\frac {3 \numA \numS^2 H \epsilon^4}{\delta'}\right)} \right)^2 \leq c_\epsilon
\epsilon + ( c_\epsilon \epsilon + \frac{C'}{\sqrt{\numS}} \sqrt{J(s,a,t)})^2$ using
Lemma~\ref{lem:expsqVdiff}.     To keep the notation concise, we use here the
shorthand $J(s,a,t) = \frac{1}{n_{tk}(s, a)}\left( \llnp(n_{tk}(s,a)) + \frac 1
2 \ln \frac{3e^4 \numS^2 \numA H}{\delta'}\right)$.  This bound holds on all
friendly episodes except at most
$
\left( 32 \numA \numS H^2+ 48 \numA \numS^2 H^3+ \numA \numS^2 H^4 + 16 \numA \numS^2\right)  \polylog(\numS, \numA, H, 1/\delta, 1/\epsilon)
$
.
Plugging this into the bound from above, we get the upper bound
\begin{align}
    &\sum_{t=1}^H \sum_{s, a \in L_{tk}} w_{tk}(s, a)\sqrt{4\numS J(s,a,t) 
        \left( c_\epsilon \epsilon + ( c_\epsilon \epsilon + C' \sqrt{J(s,a,t) / \numS})^2 \right) } 
     + \sum_{t=1}^H \sum_{s, a \in L_{tk}} 2 w_{tk}(s, a) H \numS J(s,a,t)\\
     &\leq \sum_{t=1}^H \sum_{s, a \in L_{tk}} w_{tk}(s, a)\sqrt{4\numS J(s,a,t) c_\epsilon \epsilon}
+\sum_{t=1}^H \sum_{s, a \in L_{tk}} w_{tk}(s, a)\sqrt{4\numS J(s,a,t) 
        ( c_\epsilon \epsilon + C' \sqrt{J(s,a,t) / \numS})^2 } \\
        &+ \sum_{t=1}^H \sum_{s, a \in L_{tk}} 2 w_{tk}(s, a) H \numS J(s,a,t)\\
        = & \sum_{t=1}^H \sum_{s, a \in L_{tk}} w_{tk}(s, a) \sqrt{4(c_\epsilon \epsilon + c_\epsilon^2 \epsilon^2) \numS J(s,a,t)}
             + \sum_{t=1}^H \sum_{s, a \in L_{tk}} 2 w_{tk}(s, a)  J(s,a,t) (C' + \numS H),
    \end{align}
where we used $\sqrt{a+b} \leq \sqrt{a} + \sqrt{b}$.
We now bound the first term using Lemma~\ref{lem:mainratelemma} with $r = 2, \epsilon' = \epsilon / 6, D = \frac 1 2 \ln \frac{3 e^4 \numS^2 \numA H}{\delta'}, C = 4 (c_\epsilon \epsilon + c_\epsilon^2 \epsilon^2) \numS$ on
all but $\frac{8 C \numA \numS H^2 }{\epsilon'^2} \polylog(\dots) = \frac{192
c_\epsilon (1 + c_\epsilon \epsilon) \numA \numS^2 H^2 }{\epsilon}
\polylog(\dots)$ friendly episodes by $\epsilon / 6$. 

Applying
Lemma~\ref{lem:mainratelemma} with $r = 1, \epsilon' = \epsilon / 6, D =\frac 1 2 \ln \frac{3 e^4 \numS^2 \numA H}{\delta'}$ and
$C = 2(C' + H\numS)$, we can bound the second term by
$\epsilon / 6$ on all but 
$\frac{ 8 C \numA \numS H}{\epsilon'} \polylog(\dots) = \frac{ 96 \numA \numS (C' + H \numS) H^2 }{\epsilon} \polylog(\dots)$ friendly episodes.
Hence, it holds
    \begin{align}
        \sum_{t=1}^H \sum_{s, a \in L_{tk}} w_{tk}(s, a) |(\hat P_k(s, a, t) - P(s, a, t))^\top (\tilde V^{\pi_k}_{t+1} - V^{\star}_{t+1})| \leq \frac \epsilon 3
    \end{align}
    on all friendly episodes except at most
    \begin{align}
        \bigg(&
\frac{ 96 \numA \numS (C' + H \numS) H^2 }{\epsilon}
+ 
\frac{192
c_\epsilon (1 + c_\epsilon \epsilon) \numA \numS^2 H^2 }{\epsilon}
\\ &+ 
32 \numA \numS H^2+ 48 \numA \numS^2 H^3+ \numA \numS^2 H^4 + 16 \numA \numS^2
\bigg)\polylog(\numS, \numA, H, 1/\delta, 1/\epsilon)
    \end{align}
    episodes. Since $C' = \polylog(\numS, \numA, H, 1/\delta, 1/\epsilon)$, this simplifies to
    \begin{align}
        \bigg(&
        \frac{ 96 \numA \numS}{\epsilon} + \frac{96 \numA \numS^2 H^3 }{\epsilon}
+ 
\frac{64 \numA \numS^2 H^2}{\epsilon}
+ 64 \numA \numS^2 H^2\\
&+
32 \numA \numS H^2+ 48 \numA \numS^2 H^3+ \numA \numS^2 H^4 + 16 \numA \numS^2
\bigg)\polylog(\numS, \numA, H, 1/\delta, 1/\epsilon)\\
\leq & 
((64 + 32 + 48 + 1 + 16)\numA \numS^2 H^4 + \frac{96 + 96 + 64}{\epsilon} \numA \numS^2 H^3) \polylog(\numS, \numA, H, 1/\delta, 1/\epsilon)
    \end{align}
    failure episodes in $E$. We can finally bound the failure episodes by
    \begin{align}
        \frac{417 \numA \numS^2 H^4}{\epsilon}  \polylog(\numS, \numA, H, 1/\delta, 1/\epsilon).
    \end{align}
\end{proof}

\begin{lem}
    On the good event $F^c$ for any $s \in \statespace$, $a \in \actionspace$ and $t \in [H]$ with $\delta' \leq \frac{3 \numA \numS^2 H}{e^2}$ it holds
    \begin{align}
P( s, a, t)^\top
    (\tilde V^{\pi_k}_{t+1} - V^{\pi_k}_{t+1})^2
        \leq
        &
c_\epsilon \epsilon + 
    \left(
        c_\epsilon \epsilon + 
    \sqrt{\frac{1}{n_{tk}(s,a)\numS} \left( \llnp(n_{tk}(s,a) + \frac 1 2 \ln \frac {3 \numA \numS^2 H \epsilon^4}{\delta'}\right)} \right)^2
    \end{align}
    where $C' = 1 + \sqrt{ \frac 1 2 \ln \frac{3 e^2 \numS^2 \numA
H}{\delta'}}$
    on all friendly episodes except for at most
    \begin{align}
\left( 32 \numA \numS H^2+ 48 \numA \numS^2 H^3+ \numA \numS^2 H^4 + 16 \numA \numS^2\right)  \polylog(\numS, \numA, H, 1/\delta, 1/\epsilon)
    \end{align}
    episodes.
    \label{lem:expsqVdiff}
\end{lem}
\begin{proof}
    Define $L' = \{s': w_{tk}^{t+1}(s', a' | s, a) > \wminc\}$ and $J(s') =
    \frac{ \llnp n_{t+1k}(s',a') + \frac 1 2 \ln \frac{3 e^2 \numS^2 \numA
    H}{\delta'}}{n_{t+1k}(s',a')}$ where $a' = \pi_k(s', t+1)$ and $C' = 1 + \sqrt{ \frac 1 2 \ln \frac{3 e^2 \numS^2 \numA
H}{\delta'}}$.
    Using Lemma~\ref{lem:Vdiffbound}, we bound
    \begin{align}
 & P( s, a, t)^\top
    (\tilde V^{\pi_k}_{t+1} - V^{\pi_k}_{t+1})^2
    = 
    \sum_{s'} P(s' | s, a, t)
    (\tilde V^{\pi_k}_{t+1}(s') - V^{\pi_k}_{t+1}(s'))^2\\
    \leq & \numS w_{\min}H^2 + 
    \sum_{s' \in L'} P(s' | s, a, t)
    \left(c_\epsilon \epsilon + C'\sqrt{ J(s')}\right)^2\\
\leq & c_\epsilon \epsilon + 
    C'^2\sum_{s' \in L'} P(s' | s, a, t) J(s')
    + c_\epsilon^2 \epsilon^2
    + 2c_\epsilon \epsilon C' \sum_{s' \in L'} P(s' | s, a, t)\sqrt{J(s')}
\end{align}
on all friendly episodes except at most 
$\left( 32 + 48 \numS H + \numS H^2 \right) \numA \numS H^2 \polylog(\numS, \numA, H, 1/\delta, 1/\epsilon)$. 
Define now $L'' = \{(s', a') \, : \, s' \in L', a' = \pi_k(s', t+1)\}$. 
We apply Lemma~\ref{lem:ratelemmacond} with $|T| = \{t+1\}, C = 1, D = \frac 1 2 \ln \frac {3 e^2 \numS^2 \numA H}{\delta'} \geq 1, r = 1$ and $\epsilon' = 1 / \numS$ to
\begin{align}
\sum_{s' \in L'} P(s' | s, a, t) J(s')
=&
    \sum_{s',a' \in L''}  \frac{w_{tk}^{t+1}(s', a' | s, a)}{n_{t+1k}(s', a')}\left(\llnp n_{t+1k}(s',a') + \frac 1 2 \ln \frac{3 e^2 \numS^2 \numA
        H}{\delta'}\right)\\
        \leq & 
        \frac{1}{n_{tk}(s,a) \numS} \left(\llnp(n_{tk}(s,a)) + \frac 1 2 \ln \frac{3 \numS^2 \numA H e^4}{\delta'}\right)
\end{align}
on all but at most $8 \numA \numS^2 \polylog(\numA, \numS, H, 1/\delta, 1/\epsilon)$ friendly episodes.
Similarly, we bound
\begin{align}
    &\sum_{s' \in L'} P(s' | s, a, t) \sqrt{J(s')}\\
=&
    \sum_{s',a' \in L''}w_{tk}^{t+1}(s', a' | s, a)  \sqrt{\frac{1}{n_{t+1k}(s', a')}\left(\llnp n_{t+1k}(s',a') + \frac 1 2 \ln \frac{3 e^2 \numS^2 \numA
    H}{\delta'}\right)}\\
    \leq & \sqrt{
\frac{1}{n_{tk}(s,a) \numS} \left(\llnp(n_{tk}(s,a)) + \frac 1 2 \ln \frac{3 \numS^2 \numA H e^4}{\delta'}\right)}
\end{align}
on all but at most $8 \numA \numS^2 \polylog(\numA, \numS, H, 1/\delta, 1/\epsilon)$ friendly episodes.
Hence on all friendly episodes except those failure episodes, we get 
    \begin{align}
 P( s, a, t)^\top
    (\tilde V^{\pi_k}_{t+1} - V^{\pi_k}_{t+1})^2
    \leq & c_\epsilon \epsilon + 
    \left(
        c_\epsilon \epsilon + 
    \sqrt{\frac{C'^2}{n_{tk}(s,a)\numS} \left( \llnp(n_{tk}(s,a) + \frac 1 2 \ln \frac {3 \numA \numS^2 H \epsilon^4}{\delta'}\right)} \right)^2.
\end{align}
\end{proof}

\begin{lem}
    Consider a fix $s' \in \statespace$ and $t \in [H]$, $\delta' \leq \frac{3 \numA \numS^2 H}{e^2}$ and the good event $F^c$.
    On all but at most 
    \begin{align}
        \left( 32 + 48 \numS H + \numS H^2 \right) \numA \numS H^2 \polylog(\numS, \numA, H, 1/\delta, 1/\epsilon)
    \end{align}
    friendly episodes $E$ it holds that
    \begin{align}
        V^{\pi_k}_{t}(s') - \tilde V^{\pi_k}_{t}(s')
        \leq& c_\epsilon \epsilon + \left( 1 + \sqrt{1 2 \ln \frac{3 e^2 \numS^2 \numA
H}{\delta'}} \right)
\sqrt{\frac{1}{n_{tk}(s',a')} \left(\llnp n_{tk}(s',a') + \frac 1 2 \ln \frac{3 e^2 \numS^2 \numA
H}{\delta'}\right)},
    \end{align}
    where $a' = \pi_k(s',t)$.
    \label{lem:Vdiffbound}
\end{lem}
\begin{proof}
    For any $t$,$s'$ and $a' = \pi_k(s', t)$ we use Lemma~\ref{lem:valuediff} to write the value difference as
\begin{align}
    \tilde V^{\pi_k}_{t}(s') - V^{\pi_k}_{t}(s')
=& \sum_{u=t}^H \sum_{s,a} w_{tk}^u(s, a| s', a') (\tilde P_k(s, a, u) - P(s, a, u))^\top \tilde V_{u+1}\\
    &+ \sum_{u=t}^H \sum_{s,a} w_{tk}^u(s, a| s', a') (\tilde r_k(s,a,u) - r(s,a, u))
\end{align}
    Let $L_k^{ut} = \{ s,a \in \saspace \, : \, w^u_{tk}(s, a | s', a') \geq w_{\min} \}$ be the set of state-action pairs for which the conditional probability of observing is sufficiently large.
    Then we can bound the low-probability differences as
    \begin{align}
        & \sum_{u=t}^H \sum_{s,a \in (L_k^{ut})^c} w_{tk}^u(s, a| s', a') [(\tilde r_k(s,a,u) - r(s,a, u)) + (P(s, a, u) - \tilde P_k(s, a, u))^\top \tilde V_{u+1}]\\
        \leq & \sum_{u=t}^H \sum_{s,a \in (L_k^{ut})^c} w_{\min} H 
        \leq w_{\min} H^2 \numS = c_\epsilon \epsilon.
    \end{align}
    For the other terms with significant conditional probability, we can leverage the fact that we only consider events in $(F_k^{R})^c$ and $(F_k^{P})^c$ to bound
    \begin{align}
        & \sum_{u=t}^H \sum_{s,a \in L_k^{ut}} w_{tk}^u(s, a| s', a') (\tilde r_k(s,a,u) - r(s,a, u)) \\
        \leq & \sum_{u=t}^H \sum_{s,a \in L_k^{ut}} w_{tk}^u(s, a| s', a') \sqrt{\frac{32} {n_{tk}(s, a)} 
        \left(\llnp(n_{tk}(s, a))+ \frac 1 2\ln \frac{3 \numS \numA H}{\delta'}\right)}
    \end{align}
    and 
    \begin{align}
        &\sum_{u=t}^H \sum_{s,a \in L_k^{ut}} w_{tk}^u(s, a| s', a') (P(s, a, u) - \tilde P_k(s, a, u))^\top \tilde V_{u+1}\\
        \leq & 
        \sum_{u=t}^H \sum_{s,a \in L_k^{ut}} w_{tk}^u(s, a| s', a') \sum_{s''} \tilde V_{u+1}(s'') 
            \sqrt{\frac{2 P(s'' | s,a ,u)}{n_{uk}(s, a)} 
                \left(2 \llnp(n_{uk}(s, a))+ \ln \frac{3 \numS^2 \numA H}{\delta'}\right)
    }\\
    & + \sum_{u=t}^H \sum_{s,a \in L_k^{ut}} w_{tk}^u(s, a| s', a') \sum_{s''}  
    \frac{\tilde V_{u+1}(s'')}{n_{uk}(s, a)} 
    \left(2 \llnp(n_{uk}(s, a))+ \ln \frac{3 \numS^2 \numA H}{\delta'}\right)       \\
        \leq&  
        \sum_{u=t}^H \sum_{s,a \in L_k^{ut}} w_{tk}^u(s, a| s', a') 
            \sqrt{\frac{2 \numS H^2}{n_{uk}(s, a)} 
                \left(2 \llnp(n_{uk}(s, a))+ \ln \frac{3 \numS^2 \numA H}{\delta'}\right)
    }\\
    & + \sum_{u=t}^H \sum_{s,a \in L_k^{ut}} w_{tk}^u(s, a| s', a')   
    \frac{\numS H}{n_{uk}(s, a)} 
    \left(2 \llnp(n_{uk}(s, a))+ \ln \frac{3 \numS^2 \numA H}{\delta'}\right)
\end{align}
where we use Cauchy Schwarz for the last inequality.
Combining these individual bounds, we can upper-bound the value difference as
\begin{align}
    & \tilde V^{\pi_k}_{t}(s') - V^{\pi_k}_{t}(s')\\
 \leq & c_\epsilon \epsilon + 
        \sum_{u=t}^H \sum_{s,a \in L_k^{ut}} w_{tk}^u(s, a| s', a') 
        \sqrt{\frac{(4 \sqrt 2 + 2 \sqrt{\numS} H)^2}{n_{uk}(s, a)} 
                \left(\llnp(n_{uk}(s, a))+ \frac 1 2\ln \frac{3 \numS^2 \numA H}{\delta'}\right)
    }\\
    & + \sum_{u=t}^H \sum_{s,a \in L_k^{ut}} w_{tk}^u(s, a| s', a')   
    \frac{2\numS H}{n_{uk}(s, a)} 
    \left(\llnp(n_{uk}(s, a))+ \frac 1 2 \ln \frac{3 \numS^2 \numA H}{\delta'}\right)
    \label{eqn:123a}
\end{align}

We now apply Lemma~\ref{lem:ratelemmacond} with
$r = 2, D = \frac 1 2 \ln \frac {3 \numS^2 \numA H}{\delta'}, C = (4 \sqrt 2 + 2 \sqrt{\numS} H)^2, T=\{t+1, t+2, \dots H\}$
and $\epsilon' = 1$ and get that
the second term above is bounded by 
\begin{align}
    \sqrt{\frac{1}{n_{tk}(s',a')}\left(\llnp n_{tk}(s',a') + \frac 1 2 \ln \frac{3 e^2 \numS^2 \numA
H}{\delta'}\right)}
\end{align}
on all friendly episodes but at most 
\begin{align}
    \frac{8C \numA \numS
H^2}{\epsilon'^2}  \polylog(\numS, \numA, H, 1/\delta, 1/\epsilon)=  
(32 + 16 \sqrt{2\numS} H + \numS H^2 )\numA \numS H^2 \polylog(\numS, \numA, H, 1/\delta, 1/\epsilon)
\end{align}
episodes.
We apply Lemma~\ref{lem:ratelemmacond} again to the final term in Equation~\eqref{eqn:123a} above with $r = 1, D = \frac 1 2 \ln \frac{3 \numS^2 \numA H}{\delta'} \geq 1,T=\{t+1, t+2, \dots H\}, C = 2 \numS H$ and $\epsilon' = 1$.
Then the final term is bounded by $
\frac{1}{n_{tk}(s',a')} \left( \llnp n_{tk}(s',a') + \frac 1 2 \ln \frac{3 e^2 \numS^2 \numA H}{\delta'} \right)$.
on all friendly episodes  but 
\begin{align}
    \frac{8 C \numA \numS
H^2}{\epsilon'^2}  \polylog(\numS, \numA, H, 1/\delta, 1/\epsilon) = 
16 \numA \numS^2 H^3 \polylog(\numS, \numA, H, 1/\delta, 1/\epsilon)
\end{align}
many.
Combining these bounds, we arrive at
    \begin{align}
        & V^{\pi_k}_{t}(s') - \tilde V^{\pi_k}_{t}(s')\\
        \leq& c_\epsilon \epsilon + \sqrt{\frac{1}{n_{tk}(s',a')}
\left(\llnp n_{tk}(s',a') + \frac 1 2 \ln \frac{3 e^2 \numS^2 \numA
H}{\delta'}\right)}\\ &+ 
\frac{1}{n_{tk}(s',a')} \left( \llnp n_{tk}(s',a') + \frac 1 2 \ln \frac{3 e^2 \numS^2 \numA H}{\delta'} \right)\\
\leq & c_\epsilon \epsilon + \left( 1 + \sqrt{ \frac{1}{2} \ln \frac{3 e^2 \numS^2 \numA
H}{\delta'}} \right)
\sqrt{\frac{1}{n_{tk}(s',a')} \left(\llnp n_{tk}(s',a') + \frac 1 2 \ln \frac{3 e^2 \numS^2 \numA
H}{\delta'}\right)},
    \end{align}
    where we bounded $\sqrt{\frac{1}{n_{tk}(s',a')} \left(\llnp n_{tk}(s',a') + \frac 1 2 \ln \frac{3 e^2 \numS^2 \numA
    H}{\delta'}\right)}$ by $\frac 1 2 \ln \frac{3 e^2 \numS^2 \numA H}{\delta'}$ since it is decreasing in $n_{tk}(s',a')$ and we therefore can simply use $n_{tk}(s',a') = 1$ (entire bound holds trivially for $n_{tk}(s',a') = 0$).
\end{proof}

\subsection{Useful Lemmas}

\begin{lem}[Value Difference Lemma]
For any two MDPs $M'$ and $M''$ with rewards $r'$ and $r''$ and transition probabilities $P'$ and $P''$, the difference in values with respect to the same policy $\pi$ can be written as
\begin{align}
V'_i(s) - V''_i(s) = 
\Ex'' \left[\sum_{t=i}^H (r'(s_t, a_t, t) - r''(s_t, a_t, t)) \bigg| s_i = s\right]
+\Ex'' \left[\sum_{t=i}^H (P'(s_t, a_t, t) - P''(s_t, a_t, t))^\top V'_{t+1} \bigg| s_i = s\right]
\end{align}
where $V'_{H+1} = V''_{H+1} = \vec 0$ and the expectation $\Ex'$ is taken w.r.t to $P'$ and $\pi$ and $\Ex''$ w.r.t. $P''$ and $\pi$.
\label{lem:valuediff}
\end{lem}
\begin{proof}
    For $i = H+1$ the statement is trivially true. We assume now it holds for $i+1$ and show it holds also for $i$.
Using only this induction hypothesis and basic algebra, we can write
    \begin{align}
& V'_i(s) - V''_i(s)\\
= & \Ex_\pi[ r'(s_i,a_i,i) + {V'_{i+1}}^\top P'(s_i,a_i,i) - r''(s_i,a_i,i) - {V''_{i+1}}^\top P''(s_i,a_i,i) | s_i = s] \\
= & \Ex_\pi[r'(s_i,a_i,i) - r''(s_i,a_i,i) | s_i = s]  
  + \Ex_\pi\left[\sum_{s' \in \statespace} V'_{i+1}(s') (P'(s' |s_i,a_i,i) - P''(s' | s_i,a_i,i)) \bigg| s_i = s\right]\\
  & + \Ex_\pi\left[\sum_{s' \in \statespace} P''(s' | s_i,a_i,i) (V'_{i+1}(s') - V''_{i+1}(s')) \bigg| s_i = s\right]\\
= & \Ex_\pi[r'(s_i,a_i,i) - r''(s_i,a_i,i) | s_i = s]  
  + \Ex_\pi\left[\sum_{s' \in \statespace} V'_{i+1}(s') (P'(s' |s_i,a_i,i) - P''(s' | s_i,a_i,i)) \bigg| s_i = s\right]\\
  & + \Ex''\left[V'_{i+1}(s_{i+1}) - V''_{i+1}(s_{i+1})) \bigg| s_i = s\right]\\
= & \Ex_\pi[r'(s_i,a_i,i) - r''(s_i,a_i,i) | s_i = s]  
  + \Ex_\pi\left[\sum_{s' \in \statespace} V'_{i+1}(s') (P'(s' |s_i,a_i,i) - P''(s' | s_i,a_i,i)) \bigg| s_i = s\right]\\
  & + \Ex''\left[
\Ex'' \left[\sum_{t=i+1}^H (r'(s_t, a_t, t) - r''(s_t, a_t, t)) \bigg| s_{i+1} \right]
+\Ex'' \left[\sum_{t=i+1}^H (P'(s_t, a_t, t) - P''(s_t, a_t, t))^\top V'_{t+1} \bigg| s_{i+1} \right]
 \bigg| s_i = s\right]\\
= & 
\Ex'' \left[\sum_{t=i}^H (r'(s_t, a_t, t) - r''(s_t, a_t, t)) \bigg| s_i = s\right]
+\Ex'' \left[\sum_{t=i}^H (P'(s_t, a_t, t) - P''(s_t, a_t, t))^\top V'_{t+1} \bigg| s_i = s\right]
    \end{align}
where the last equality follows from law of total expectation
\end{proof}
\begin{lem}[Algorithm ensures optimism]
    On the good event $F^c$ it holds that for all episodes $k$, $t \in [H]$, $s \in \statespace$ that
    \begin{align}
        V^{\pi_k}_t(s) \leq V^{\star}_t(s) \leq \tilde V^{\pi_k}_t(s).
    \end{align}
    \label{lem:optimism}
\end{lem}
\begin{proof}
The first inequality follows simply from the definition of the optimal value
function $V^{\star}$. 

Since all outcome we consider are in the event $(F^V_k)^c$, we know that the
true transition probabilities $P$, the optimal policy $\pi^{\star}$ and optimal
policy $V^{\star}$ are a feasible solution for the optimistic planning problem in Lemma~\ref{lem:optplaninterpret} that \algname solves.
It therefore follows immediately that $p_0^\top \tilde V^{\pi_k}_1 \geq p_0^\top V^{\star}_1$.
\end{proof}

\section{General Concentration Bounds}
\label{sec:concentrationproofs}
\begin{lem}
Let $X_1,X_2,\ldots$ be a martingale difference sequence adapted to filtration $\set{\mathcal F_t}_{t=1}^\infty$ with $X_t$ conditionally $\sigma^2$-subgaussian
    so that $\Ex[\exp(\lambda(X_t-\mu))|\mathcal F_{t-1}] \leq\exp(\lambda^2 \sigma^2 / 2)$ almost surely for all $\lambda \in \reals$.
Then with $\hat \mu_t = \frac{1}{t} \sum_{i=1}^t X_i$ we have for all $\delta \in (0, 1]$
\begin{align}
    \prob\left(\exists t : \left|\hat \mu_t - \mu\right| \geq \sqrt{\frac{4\sigma^2}{t} \left(2\llnp(t) + \ln \frac{3}{\delta}\right)}\right) \leq 2\delta\,.
\end{align}
    \label{lem:uniformhoeffding}
\end{lem}

\begin{proof}
Let $S_t = \sum_{s=1}^t (X_s - \mu)$. Then
\begin{align}
    & \prob\left(\exists t : \hat \mu_t - \mu \geq \sqrt{\frac{4\sigma^2}{t} \left(2\llnp (t) + \ln \frac{3}{\delta}\right)}\right) \\
    \leq& \prob\left(\exists t : S_t \geq \sqrt{4\sigma^2 t \left(2\llnp(t)  + \ln \frac{3}{\delta}\right)}\right) \\
    \leq& \sum_{k=0}^\infty \prob \left(\exists t \in [2^k, 2^{k+1}] : S_t \geq \sqrt{4\sigma^2 t \left(2\llnp(t) + \ln \frac{3}{\delta}\right)} \right) \\
    \leq& \sum_{k=0}^\infty \prob \left(\exists t \leq 2^{k+1} : S_t \geq \sqrt{2 \sigma^2 2^{k+1} \left(2\llnp(2^k) + \ln \frac{3}{\delta}\right)} \right)
\end{align}
We now consider $M_t = \exp(\lambda S_t)$ for $\lambda > 0$ which is a nonnegative sub-martingale and use the short-hand $f = \sqrt{2 \sigma^2 2^{k+1} \left(2\llnp(2^k) + \ln \frac{3}{\delta}\right)}$.
Then by Doob's maximal inequality for nonnegative submartingales
    \begin{align}
        \prob\left( \exists t \leq 2^{k+1} : S_t \geq f\right)
        = \prob\left( \max_{t \leq 2^{k+1}} M_t \geq \exp(\lambda f) \right)
        \leq \frac{ \Ex[M_{2^{k+1}}] }{\exp(\lambda f)}
        \leq \exp\left( 2^{k+1} \frac{\lambda^2 \sigma^2}{2} - \lambda f \right).
    \end{align}
    Choosing the optimal $\lambda = \frac{f}{\sigma^2 2^{k+1}}$ we obtain the bound
    \begin{align}
        \prob\left( \exists t \leq 2^{k+1} : S_t \geq f\right)
        \leq& \exp \left( -\frac{f^2}{2^{k+2} \sigma^2} \right)
        = \exp \left(- 2 \llnp(2^k) - \ln \frac 3 \delta \right)
        = \frac \delta 3 \exp \left(-2 \llnp(2^k) \right)\label{eqn:exp-glem1} \\
        =& \frac \delta 3 \exp \left( - \max\{0, 2\ln \max\{0, \ln 2^k \}\} \right)
        = \frac \delta 3 \min \left\{ 1, (k\ln 2)^{-2} \right\}\\
        \leq & \frac \delta 3 \min \left\{ 1, \frac{1}{k^2 \ln 2} \right\}.
    \end{align}
    Plugging this back in the bound from above, we get
\begin{align}
     \prob\left(\exists t : \hat \mu_t - \mu \geq \sqrt{\frac{4\sigma^2}{t} \left(2\llnp (t) + \ln \frac{3}{\delta}\right)}\right) 
    \leq& \frac{\delta}{3} \sum_{k=0}^\infty \min\set{1, \frac{1}{k^2 \ln(2)}} \\
    =& \delta\ \frac{1}{3}\left( \frac{\pi^2}{6 \ln 2} + 2 - 1 / \ln(2)\right) \leq \delta\,.
    \label{eqn:exp-glem1fin}
\end{align}
For the other side, the argument follows completely analogously with
\begin{align}
    \prob\left( \exists t \leq 2^{k+1} : S_t \leq -f \right)
    =& \prob\left( \exists t \leq 2^{k+1} : -S_t \geq f \right)
\\
    =& \prob\left( \max_{t \leq 2^{k+1}} \exp(- \lambda S_t) \geq \exp(\lambda f) \right)
    \\
    \leq & \frac{ \Ex[\exp(-\lambda S_{2^{k+1}})] }{\exp(\lambda f)}
    \leq \exp\left( 2^{k+1} \frac{\lambda^2 \sigma^2}{2} - \lambda f \right).
\end{align}

\end{proof}

\begin{lem}
Let $X_1,X_2,\ldots$ be a sequence of Bernoulli random variables with bias $\mu \in [0,1]$. 
Then for all $\delta \in (0, 1]$
\begin{align}
    \prob \left(\exists t : \left|\hat \mu_t - \mu\right| 
    \geq \sqrt{\frac{2\mu}{t} \left( 2\llnp(t) + \ln \frac{3}{\delta}\right)} + \frac 1 t \left(2\llnp(t) + \ln \frac 3 \delta \right)
    \right) \leq 2\delta
\end{align}
    \label{lem:uniformbern}
\end{lem}
\begin{proof}

    \begin{align}
        &\prob \left(\exists t : \hat \mu_t - \mu \geq
        \sqrt{\frac{2\mu}{t} \left( 2\llnp(t) + \ln \frac{3}{\delta}\right)} + \frac 1 t \left(2\llnp(t) + \ln \frac 3 \delta \right)
    \right) \\
        = &
        \prob \left(\exists t : S_t \geq
        \sqrt{2\mu t \left( 2\llnp(t) + \ln \frac{3}{\delta}\right)} + 2\llnp(t) + \ln \frac 3 \delta
    \right) \\
        \leq & \sum_{k=0}^\infty
        \prob \left(\exists t \leq 2^{k+1} : S_t \geq 
        \sqrt{2\mu 2^k \left( 2\llnp(2^k) + \ln \frac{3}{\delta}\right)} + 2\llnp(2^k) + \ln \frac 3 \delta
    \right) 
    \end{align}
    Let $g = 2\llnp(2^k) + \ln \frac 3 \delta$ and $f = \sqrt{2^{k+1} \mu g} + g$.
    Further define $S_t = \sum_{i=1}^t X_i - t\mu$ and $M_t = \exp(\lambda S_t)$ which is by construction a nonnegative submartingale. 
Applying Doob's maximal inequality for nonnegative submartingales, we bound
    \begin{align}
        \prob\left( \exists t \leq 2^{k+1} : S_t \geq f\right)
        = \prob\left( \max_{i\leq 2^{k+1}} M_i \geq \exp(\lambda f) \right)
        \leq \frac{ \Ex[M_{2^{k+1}}] }{\exp(\lambda f)}
        = \exp\left( \ln \Ex[M_{2^{k+1}}] - \lambda f \right).
    \end{align}
    Since this holds for all $\lambda \in \reals$, we can bound
    \begin{align}
        \prob\left( \exists t \leq 2^{k+1} : S_t \geq f\right)
        \leq \exp\left( - \sup_{\lambda \in \reals} \left( \lambda f - \ln \Ex[M_{2^{k+1}}]\right)\right)
    \end{align}
    and using Corollary~2.11 by \citet{Boucheron2013} (see also note below proof of Corollary~2.11) bound that by
    \begin{align}
        \exp\left( - \frac{f^2}{2( 2^{k+1}\mu + f/3)}\right) 
    \end{align}
    We now argue that this quantity can be upper-bounded by $\exp(-g)$. This is equivalent to
    \begin{align}
        - \frac{f^2}{2( 2^{k+1}\mu + f/3)} \leq &  -g\\
        f^2 \geq & 2g( 2^{k+1}\mu + f/3) = \frac 2 3 gf + \frac{2^{k+2}}{3} \mu g\\
        g^2 + 2 \sqrt{2^{k+1} \mu g}g +  2^{k+1} \mu g  \geq &  \frac 2 3 g^2 + \frac 2 3 \sqrt{2^{k+1} \mu g} g + \frac{2^{k+2}}{3} \mu g\\
        \frac 1 3 g^2 +  \frac 4 3 \sqrt{2^{k+1} \mu g}g + \frac 1 3 2^{k+1} \mu g  \geq &  0.
    \end{align}
Each line is an equivalent inequality since $g, f \geq 0$ and each term on the left in the final inequality is nonnegative.
Hence, we get $
        \prob\left( \exists t \leq 2^{k+1} : S_t \geq f\right) \leq \exp(-g)$.
    Following now the arguments from the proof of Lemma~\ref{lem:uniformhoeffding} in Equations~\eqref{eqn:exp-glem1}--\eqref{eqn:exp-glem1fin}, we obtain that
    \begin{align}
        \prob \left(\exists t : \hat \mu_t - \mu \geq
        \sqrt{\frac{2\mu}{t} \left( 2\llnp(t) + \ln \frac{3}{\delta}\right)} + \frac 1 t \left(2\llnp(t) + \ln \frac 3 \delta \right)
    \right) \leq \delta.
    \end{align}
For the other direction, we proceed analogously to above and arrive at
    \begin{align}
        \prob\left( \exists t \leq 2^{k+1} : -S_t \geq f\right)
        \leq \exp\left( - \sup_{\lambda \in \reals} \left( -\lambda f - \ln \Ex[M_{2^{k+1}}]\right)\right)
    \end{align}
which we bound similarly to above by
    \begin{align}
        \exp\left( - \frac{f^2}{2( 2^{k+1}\mu - f/3)}\right) \leq
        \exp\left( - \frac{f^2}{2( 2^{k+1}\mu + f/3)}\right) \leq \exp(-g).
    \end{align}
\end{proof}

\begin{lem}[Uniform L1-Deviation Bound for Empirical Distribution]
    Let $X_1, X_2, \dots$ be a sequence of i.i.d. categorical variables on $[U]$ with distribution $P$.
Then for all $\delta \in (0, 1]$
\begin{align}
\prob\left( \exists t \, : \, \|\hat P_t - P\|_1 \geq \sqrt{\frac{4}{t}\left( 2 \llnp(t) + \ln\frac{3 (2^U - 2)}{\delta} \right)}\right) \leq \delta
\end{align}
where $\hat P_t$ is the empirical distribution based on samples $X_1 \dots X_t$.
\label{lem:l1deviation}
\end{lem}
\begin{proof}
We use the identity $\|Q - P\|_1 = 2 \max_{B \subseteq \mathcal B} Q(B) - P(B)$ which
holds for all distributions $P, Q$ defined on the finite set $\mathcal B$ to bound
\begin{align}
& \prob\left( \exists t \, : \, \|\hat P_t - P\|_1 \geq \sqrt{\frac{4}{t}\left( 2 \llnp(t) + \ln\frac{3 (2^U - 2)}{\delta} \right)}\right)\\
= & \prob\left( \max_{t, B \subseteq [U]} \hat P_t(B) - P(B) \geq \frac 1 2 \sqrt{\frac{4}{t}\left( 2 \llnp(t) + \ln\frac{3 (2^U - 2)}{\delta} \right)}\right)\\
\leq & \sum_{B \subseteq [U]} \prob\left( \max_{t} \hat P_t(B) - P(B) \geq \sqrt{\frac{1}{t}\left( 2 \llnp(t) + \ln\frac{3 (2^U - 2)}{\delta} \right)}\right).
\end{align}
Define now $S_t = \sum_{i=1}^t \indicator{X_1 \in B} - tP(B)$ which is a martingale sequence. Then the last line above is equivalent to
\begin{align}
& \sum_{B \subseteq [U]} \prob\left( \max_{t} S_t \geq \sqrt{t\left( 2 \llnp(t) + \ln\frac{3 (2^U - 2)}{\delta} \right)}\right)\\
\leq & 
 \sum_{B \subseteq [U]} \prob\left( \max_{k \in \mathbb N, t \in [2^{k}, 2^{k+1}]} S_t \geq \sqrt{t\left( 2 \llnp(t) + \ln\frac{3 (2^U - 2)}{\delta} \right)}\right)\\
\leq & \sum_{B \subseteq [U]} \sum_{k=0}^\infty \prob\left( \max_{t \in [2^{k}, 2^{k+1}]} S_t \geq \sqrt{t\left( 2 \llnp(t) + \ln\frac{3 (2^U - 2)}{\delta} \right)}\right)\\
\leq & \sum_{B \subseteq [U]} \sum_{k=0}^\infty \prob\left( \max_{t \leq 2^{k+1}} S_t \geq \sqrt{2^k \left( 2 \llnp(2^k) + \ln\frac{3 (2^U - 2)}{\delta} \right)}\right)\\
     = & \sum_{B \subseteq [U]} \sum_{k=0}^\infty \prob\left( \max_{t \leq 2^{k+1}} \exp(\lambda S_t) \geq \exp(\lambda f)\right)\\
    = & \sum_{B \subseteq [U], B \neq \emptyset, B \neq [U]} \sum_{k=0}^\infty \prob\left( \max_{t \leq 2^{k+1}} \exp(\lambda S_t) \geq \exp(\lambda f)\right)
\end{align}
where
    $f = \sqrt{2^k \left( 2 \llnp(2^k) + \ln\frac{3 (2^U - 2)}{\delta} \right)}$ and $\lambda \in \reals$ and the last equality follows from the fact that for $B = \emptyset$ and $B = [U]$ the difference between the distributions has to be $0$.
Since $\indicator{X_1 \in B} - tP(B)$ is a centered Bernoulli variable it
is $1/2$-subgaussian and so $S_t$ satisfies $\Ex[\exp(\lambda S_t)] \leq
\exp(\lambda^2 t / 8)]$. Since $S_t$ is a martingale, $\exp(\lambda S_t)$
is a nonnegative sub-martingale and we can apply the maximal inequality to
bound
\begin{align}
    \prob\left( \max_{t \leq 2^{k+1}} \exp(\lambda S_t) \geq \exp(\lambda f)\right)
\leq \exp\left(\frac 1 8 \lambda^2 2^{k+1} - \lambda f\right).
\end{align}
Choosing $\lambda = \frac{4f}{2^{k+1}}$,  we get
$\prob\left( \max_{t \leq 2^{k+1}} \exp(\lambda S_t) \geq \exp(\lambda f)\right) \leq 
\exp\left(-\frac{f^2}{2^k}\right)$.
Hence, using the same steps as in the proof of Lemma~\ref{lem:uniformhoeffding}, we get
$\prob\left( \max_{t \leq 2^{k+1}} \exp(\lambda S_t) \geq \exp(\lambda f) \right)
\leq \frac {\delta}{3(2^{[U]} -2)}\min\left\{1, \frac 1 {k^2 \ln 2} \right\}$ and then
\begin{align}
    & \prob\left( \exists t \, : \, \|\hat P_t - P\|_1 \geq \sqrt{\frac{4}{t}\left( 2 \llnp(t) + \ln\frac{3 (2^U - 2)}{\delta} \right)}\right)\\
    \leq & \sum_{B \subseteq [U], B \neq \emptyset, B \neq [U]} \frac{\delta}{3(2^{[U]} - 2)} \sum_{k=0}^\infty \min\left\{1, \frac 1 {k^2 \ln 2} \right\}
    \leq  \sum_{B \subseteq [U], B \neq \emptyset, B \neq [U]} \frac{\delta}{2^{[U]} - 2} 
    = \delta.
\end{align}
\end{proof}

\begin{lem}
    Let $\mathcal F_i$ for $i=1\dots$ be a filtration and $X_1, \dots X_n$ be a sequence of Bernoulli random variables with $\prob(X_i = 1 | \mathcal F_{i-1}) = P_i$ with $P_i$ being $\mathcal F_{i-1}$-measurable and $X_i$  being $\mathcal F_{i}$ measurable.
    It holds that
    \begin{align}
        \prob \left(\exists n : \,\, \sum_{t=1}^n X_t < \sum_{t=1}^n P_t / 2 - W  \right) \leq e^{- W}
    \end{align}
    \label{lem:uniformwnconc}
\end{lem}

\begin{proof}
    $P_t - X_t$ is a Martingale difference sequence with respect to the filtration $\mathcal F_t$. 
    Since $X_t$ is nonnegative and has finite second moment, we have for any $\lambda > 0$ that
    $\Ex\left[e^{-\lambda (X_t - P_t)} | \mathcal F_{t-1} \right] \leq e^{\lambda^2 P_t / 2}$ (Exercise 2.9, \citet{Boucheron2013}). 
    Hence, we have
    \begin{align}
        \Ex\left[ e^{\lambda(P_t - X_t) - \lambda^2 P_t / 2} | \mathcal F_{t-1} \right] \leq 1
    \end{align}
    and by setting $\lambda = 1$, we see that 
    \begin{align}
        M_n = e^{\sum_{t=1}^n (-X_t + P_t / 2)}
    \end{align}
    is a supermartingale.
    It hence holds by Markov's inequality
    \begin{align}
        \prob \left( \sum_{t=1}^n (-X_t + P_t / 2) \geq W \right)
        = 
        \prob \left( M_n \geq e^W \right)
        \leq e^{-W} \Ex[M_n] \leq e^{-W}
    \end{align}
    wich gives us the derised result 
    \begin{align}
        \prob \left( \sum_{t=1}^n X_t \leq \sum_{t=1}^n P_t / 2 - W \right)
        \leq e^{-W}
    \end{align}
    for a fixed $n$.
    We define now the stopping time $\tau = \min\{ t \in \mathbb N \, : \, M_t > e^W \}$ and the sequence $\tau_n = \min\{ t \in \mathbb N \, : \, M_t > e^W \vee t \geq n \}$.
    Applying the convergence theorem for nonnegative supermartingales
    (Theorem~5.2.9 in \citet{Durrett2010}), we get that $\lim_{t \rightarrow
    \infty} M_t$ is well-defined almost surely. Therefore, $M_\tau$ is
    well-defined even when $\tau = \infty$.
    By the optional stopping theorem for nonnegative supermartingales (Theorem
    5.7.6 by \citet{Durrett2010}), we have $\Ex[M_{\tau_n}] \leq \Ex[M_0] \leq
    1$ for all $n$ and applying Fatou's lemma, we obtain 
    $\Ex[M_\tau] = \Ex[\lim_{n \rightarrow \infty} M_{\tau_n}] \leq \lim \inf_{n \rightarrow \infty} \Ex[M_{\tau_n}] \leq 1$. 
    Using Markov's inequality, we can finally bound
    \begin{align}
        \prob\left( \exists n:\,\, \sum_{t=1}^n X_t < \frac 1 2 \sum_{t=1}^n P_t - W \right)
        \leq
        \prob ( \tau < \infty) \leq \prob ( M_{\tau} > e^W) 
        \leq e^{-W} \Ex[M_{\tau}] \leq e^{-W}.
    \end{align}
\end{proof}

\end{document}